\documentclass{article} 
\usepackage[sort]{natbib}
\usepackage[margin=2cm]{geometry}

\usepackage{amsmath,amsfonts,bm}

\def\eqref#1{equation~\ref{#1}}

\def\1{\bm{1}}

\DeclareMathAlphabet{\mathsfit}{\encodingdefault}{\sfdefault}{m}{sl}
\SetMathAlphabet{\mathsfit}{bold}{\encodingdefault}{\sfdefault}{bx}{n}

\newcommand{\KL}{D_{\mathrm{KL}}}

\DeclareMathOperator*{\argmax}{arg\,max}

\usepackage{enumitem}

\usepackage{hyperref}
\usepackage{url}
\usepackage{cleveref}

\usepackage{amsmath,amssymb,amsthm}
\usepackage{bm}
\usepackage{algpseudocode}
\usepackage{algorithm}

\usepackage{graphicx}

\newtheorem{theorem}{Theorem}[section]
\newtheorem{definition}{Definition}
\newtheorem{adef}{Definition}[section]
\newtheorem{assumption}{Assumption}[section]
\crefname{assumption}{assumption}{assumptions}
\Crefname{assumption}{Assumption}{Assumptions}

\newtheorem{example}{Example}[section]
\newtheorem{lemma}{Lemma}[section]

\newtheorem*{remark}{Remark}

\newcommand{\supp}{\text{supp}}
\newcommand{\expand}[1]{#1}
\newcommand{\bb}[1]{\mathbb{#1}}
\newcommand{\bv}[1]{\boldsymbol{#1}}
\newcommand{\bvec}[1]{\boldsymbol{#1}}
\newcommand{\ca}[1]{\mathcal{#1}}

\newcommand{\dg}{\text{dg}}
\newcommand{\tr}{\text{tr }}
\newcommand{\av}{\text{av}}
\newcommand{\Diff}{\text{Diff}}
\newcommand{\Spec}{\text{Spec}}
\newcommand{\iso}{\text{iso}}
\newcommand{\tvec}{\text{vec}}

\title{Measure, Manifold, Learning, and Optimization: A Theory Of Neural Networks}

\author{Shuai Li
  \thanks{Research done during the stay at South China University of
Technology, Chinese University of Hong Kong, Microsoft Research Asia, and
University of Science and Technology of China, in reverse chronological order.}
\\
\texttt{lishuai918@gmail.com} \\
}

\begin{document}

\maketitle

\begin{abstract}
We present a formal measure-theoretical theory of neural networks (NN) built on
{\it probability coupling theory}.
Our main contributions are summarized as follows.
\begin{itemize}
\item Built on the formalism of probability coupling theory, we derive an algorithm
framework, named Hierarchical Measure Group and Approximate System (HMGAS),
nicknamed S-System, that is designed
to learn the complex hierarchical, statistical dependency in the physical world, of which
the hierarchical structure is formulated as measure-theoretical assumptions.
\item We show that NNs are special cases of S-System when the probability
kernels assume certain exponential family distributions. Activation Functions
are derived formally. We further endow geometry on NNs through {\it information
  geometry}, show that intermediate feature spaces of NNs are stochastic
manifolds, and prove that ``distance'' between samples is contracted as
layers stack up.
\item S-System shows NNs are inherently stochastic, and under a set of {\it
realistic} boundedness and diversity conditions, it enables us to prove that for
{\it large size nonlinear deep} NNs with a class of losses, including the hinge
loss, all local minima are global minima with zero loss errors, and regions
around the minima are flat basins where all eigenvalues of Hessians are concentrated around
zero, using tools and ideas from {\it mean field theory}, {\it random matrix theory},
  and {\it nonlinear operator equations}.
\item S-System, the information-geometry structure and the optimization
behaviors combined completes the analog between {\it Renormalization Group} (RG) and
NNs. It shows that a NN is a complex adaptive system that estimates the
statistic dependency of microscopic object, e.g., pixels, in multiple
scales. Unlike clear-cut physical quantity produced by RG in physics, e.g.,
temperature, NNs renormalize/recompose manifolds emerging through
learning/optimization that divide the sample space into highly semantically
meaningful groups that are dictated by supervised labels (in supervised NNs).
\end{itemize}

However, the above contributions describe the theory present in a backward way,
for that it is easy for readers to relate to the theory by debriefing what open
problems are related. In a logical way, the paper describes {\it four} parts of
the theory of NNs.
\begin{itemize}
\item {\bf S-System}, a formal measure-theoretical framework that builds a hierarchical
  hypothesis space, of which NNs are special cases. It
  is motivated by the fact that nature is a complex system built
  hierarchically, and a mechanism is needed for any agents living in it to
  recognize and predict hierarchical events happening.
\item The {\bf geometry} of S-System. The objects in the hypothesis space are
  probability measures, thus have an information-geometry structure. It
  characterizes the phenomenon that NNs compose and recompose manifolds that have
  increasingly high level semantic meaning.
\item The {\bf learning framework} of S-System. It describes the objective
  functions to identify an element of the hypothesis space by
  learning the parameters. It gives a principled derivation of back propagation, and unifies
  supervised learning and unsupervised learning in NNs.
\item The {\bf optimization} landscape of S-System. It identifies principles and
  conditions that make the non-convex optimization of the learning problem of
  S-System benign; that is, all local minima are global minima.
\end{itemize}

Help is solicited in \cref{sec:call-help}.

\end{abstract}

\pagebreak
\tableofcontents

\pagebreak
\section{Introduction}
\label{sec:introduction}

\subsection{A Theory of Neural Networks}
\label{sec:s-system-geometry}

\subsubsection{Debrief on the four parts of the theory}
\label{sec:debrief-four-parts}

We present a measure-theoretical theory of NNs. We summarize the parts of the
theory present in this paper in this section. To the best of our knowledge, we
do not find works that are intimately close to ours, and each part of the
theory has its own related works, which is present at the {\it end} of each
part.

The theory is divided into the following {\it four} parts:
\begin{itemize}
\item {\bf S-System}, a formal measure-theoretical framework that builds a hierarchical
  hypothesis space, of which NNs are special cases. It
  is motivated by the fact that nature is a complex system built
  hierarchically, and a mechanism is needed for any agents living in it to
  recognize and predict hierarchical events happening.
\item The {\bf geometry} of S-System. The objects in the hypothesis space are
  probability measures, thus have an information-geometry structure. It
  characterizes the phenomenon that NNs compose and recompose manifolds that have
  increasingly high level semantic meaning.
\item The {\bf learning framework} of S-System. It describes the objective
  functions to identify an element of the hypothesis space by
  learning the parameters. It gives a principled derivation of back propagation,
  and unifies supervised learning and unsupervised learning in NNs.
\item The {\bf optimization} landscape of S-System. It identifies principles and
  conditions that make the non-convex optimization of the learning problem of
  S-System benign; that is, all local minima are global minima.
\end{itemize}

We provide a short introduction to the ideas of the theory present, and defer a
full informal introduction to \cref{sec:complex-system}. The theory is a
principle to assign probabilities to events, e.g., predicting the event that
today would be a rainy or sunny day, given past experience. To recall what a
principle to assign probabilities is, we recall the principle of symmetry, and
the law of large number (LLN). The example of the principle of symmetry is
ubiquitous in elementary probability theory, i.e., the equal probability
assigned to the event that front, or back side is obtained when flipping a
coin. So is the law of large number, which assigns normal probability
distribution to that of the mean of a large number of random variables. However,
the way that these principles assign probabilities grounds on a certain ensemble
of repeated trials: the equal probability assigned to coin flipping is mostly
grounded on the observation on thousands of repeated trials; and the LLN grounds on averaging
repeated trials (with, or without the i.i.d. assumption). However, how can we
assign probabilities to the events that are not as simple as coin flipping, and
is not the average of an ensemble, like the one of weather prediction?

\begin{figure}[h]
  \centering
  \includegraphics[width=0.363\textwidth]{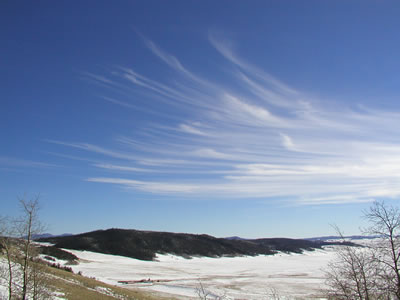}
  \includegraphics[width=0.4\textwidth]{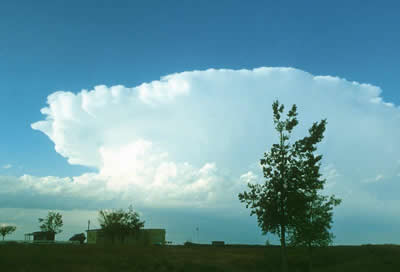}
  \caption[Cloud Types]{Cloud types: on the left is Cirrus, the type of cloud that
    normally wouldn't rain; on the right is the Cumulonimbus, the type of cloud
    that normally would lead to thunderstorms.}
  \label{fig:cloud}
\end{figure}

Given daily life experience, Cirrus (c.f. left picture of \cref{fig:cloud})
normally would not rain, but Cumulonimbus (c.f. right picture of
\cref{fig:cloud}) would lead to thunderstorms. The prediction from daily
experience comes from repeatedly observing the cloud shapes and
the weather afterwards. The physics behind the cloud formation and raining is a
complex interaction in the complex system of moisture, dust, gravity,
temperature and wind, of which the outcome is highly chaotic thus uncertain. The past
experience allows us to divide the cloud shapes into different groups by their
salient features that would have a high, or low chance to induce different
types of weather. THUS, in summary, we have a mechanism to approximate the
probability of outcomes of a complex physical system by observing certain
features of the system. That is, to estimate the plausibility of the raining
event based on the events that a certain group/type of cloud shape occurs.

The theory is a theory of the above mechanism. S-System is the mechanism to
construct feature events, e.g., the shapes of the cloud, that are divided into
groups/types and mirror the hierarchical relationship in the physical complex
system in a self-organized way, e.g., the interaction between water moleculars
and dust particles. The geometry part characterizes the information geometry
structure in the feature space. The relationship between it and S-System is
like the one between Hilbert Space and functions: it provides a manifold
structure. The learning framework characterizes how the past experience helps
learn the grouping of events: the objective function that divides events into
different cloud types/groups, and discovers feature events in a self-organized
way. The optimization part investigates how such a grouping is implementable
given objective functions in the learning framework through optimization. They
overall make up a theory of NNs (and S-System): the mechanism to
approximate/assign the probability of/to outcomes of events to estimate the
plausibility of events, e.g., the occurrence of rain or not.

\subsubsection{Debrief on insights}
\label{sec:insights}

In this section, we summarize insights of our results that are related to
existing research.

\paragraph{Inherent stochasticity of NNs.} Supervised NNs, such as Multiple
Layer Perceptron (MLP), has been normally taken as a deterministic model, a
function approximator that approximates a probability distribution. However,
underlying the superficial deterministicity, the deterministic forward
propagation maximizes the expected data likelihood, while the backward
propagation minimizes surrogates of KL divergence (details in
\cref{sec:learn-fram-htms}). Activation functions have principled derivation
(details in \cref{sec:theor-deriv-nn}), and activations in the intermediate
layers are a function of estimated realization of random variables, of which the
true probability measure is transported from the input data.  Activation
functions used in practice actually assume the exponential family
distributions, which explains why NNs learn templates, since the mean of an
exponential family distribution uniquely determines its distribution.(details
in \cref{sec:contr-effect-coarse}). S-System unifies and formalizes the
interpretations proposed in \cite{Bottou1996} \cite{Author2016}, \cite{Lin2017}
and \cite{Trevisanutto} --- the connection with these works is explained in
detail in \cref{sec:nn-hypothesis-space}.

\paragraph{Stochastic manifolds composed and recomposed by NNs.}
It is observed that NNs gradually build representations that function similar
with the ``grandmother'' cell found in a biological brain, and it has long been
hypothesized that NNs build certain kinds of manifolds. The stochasticity
discovered allows us to discover the information geometry structure in NNs, in
which a hierarchy of manifolds are composed and recomposed to build
increasingly high level semantic representations as layers stack up. More
specifically, a stochastic manifold structure is endowed on the intermediate
space of NNs, where the ``distance'' is defined to characterize the semantic
difference between events/samples. {\it As the layers go deeper, the semantic
difference potentially becomes gradually coarse-grained to reflect the higher
level semantic difference, e.g., dogs vs cats, while ignoring lower level
variations, e.g., textures (\cref{thm:cg}).} This is true due to the
information monotony phenomenon: by blocking half of the information from
propagating (using ReLU as an example), information related to irrelevant
variations could potentially be discarded. It also characterizes the phenomenon
that samples/images can be compared in different criteria in different contexts
(details in \cref{sec:contr-effect-coarse}), thus opening the new possibilities
of metric learning.

\paragraph{Symmetry in NNs}
 The event spaces of samples are the objects to study if one wants to study the
symmetry in NNs. For example, robustness to deformation in images can be
characterized as a close ``distance'' between two event collections where one
is obtained by deforming all events in the other. An example is provided in
example \ref{example:deformation}.

\paragraph{Optimization landscape of NNs.}

  The stochasticity identified enables us to analyze
NNs stochastically in its full complexity. It enables us to
characterize the optimization behaviors of NNs in \cref{thm:landscape}. It
explains the optimization myths of NNs that though being non-convex NNs can
optimize the loss to zero, and why learning progresses slowly when approaching
the minima. {\it Informally, a huge number of cooperative yet diverse neurons
can divide samples into arbitrary groups corresponding to labels.} The
assumptions made in the theorem are sufficient practice-guiding preconditions
instead of unrealistic assumptions made to make the proof work. It explains why
centering of neuron activation (\cite{Glorot2010}) is helpful, for it helps to
let the eigenvalues of the Hessian of the risk function be symmetric
w.r.t. y-axis (\cref{thm:symmetry}), thus guaranteeing the existence of
negative eigenvalues to provide loss-minimizing directions; why normalization
of neuron activation (\cite{Ioffe2015}) is helpful, for that the boundedness and
diversity conditions \ref{a:boundedness} \ref{a:diversity} ask the correlation
between neurons, formulated as cumulants, to be small, and normalization of
standard deviation possibly maintains the conditions throughout the training
(\cref{sec:cond-nice-behav}); why the larger the network is, the easier for it
to reach zero error, for our results on optimization is a
probably-approximately-correct type result where the error is controlled by the
size of the network --- the underlying reason is complicated, and we refer
readers to \cref{sec:eigenv-distr-symm}.

\paragraph{Renormalization group implemented by NNs.} NNs has been analogized
with Renormalization Group (RG) \citep{Mehta2014}. However, the analog is
incomplete for that it does not identify a large scale property, quantity, or
feature that is produced by such renormalization \citep{Lin2017}. RG is a
Nobel-prize-winning tool that bridges phenomenon across scale in complex
systems, e.g., spin glasses, in physics. In a similar vein, the theory completes the
analog, and show that NNs is a complex adaptive system that estimates the
statistic dependency of microscopic objects, e.g., pixels, in multiple
scales. Unlike clear-cut physical quantity produced by RG in physics, e.g.,
temperature, NNs renormalize/recompose manifolds emerging through
learning/optimization that divide the sample space into highly semantically
meaningful groups that are dictated by supervised labels (in supervised NNs).
 It also formalizes the folk wisdom that NNs only care
about a subset of the overall signal space, e.g., image space, in a measure
theoretical way (details in \cref{sec:phys-prob-meas}).

\paragraph{Discretization of continuous sample space done by NNs and
  formalization of semantics.}
  This is the most important one, but also perhaps the hardest one to get. {\it
S-System shows NNs work by grouping samples/events into groups --- which
emerges through optimizing an objective function --- and estimate/approximate
their true probability measure through empirical observations.} The group is a
formalization of semantics, and explains how discrete labels emerge from
continuous samples, i.e., a label identifies a group of events/samples. A
proper implementation of the above process applied hierarchically creates an
adaptive complex system that consists of a huge number of neurons. The neurons
represent different groups of events of low mutual correlation, which
preconditions the optimization results.  It gives a formalization of the
meaning of semantics in labels in supervised learning, which is a way to group
samples/events in a way meaningful to humans (details can be found in
\cref{sec:learn-fram-htms}). This is the aim of the whole paper, and is
described informally in \cref{sec:complex-system}.

\subsubsection{Roadmap of future research}
\label{sec:roadmap}

It is straightforward that ongoing efforts are working on further perfecting of
existing parts of the theory. Except for the first part, only the scaffold of the remaining three parts
is worked out. To name an example, the optimization part has worked out the
conditions for a NN with a binary loss to converge to global minima. The result
analogizes with the situation where numbers have been defined, while the algebra
on them has not. The understanding of the algebra would lead to results that
cover all kinds of losses. Besides these short term goals, we would
like to present a bigger picture that covers missing pieces, and future
extension of the theory.

\paragraph{Short term goals: adversarial robustness, generalization behaviors,
dynamical-sized self-generating networks, Lisp, C++, CUDA based NN library.}
1) The
adversarial samples phenomenon \citep{Szegedy2013} is caused by an uncontrolled
propagation of errors in a high dimensional space in an exponential way, a
proper characterization of errors may solve the phenomenon. 2) Though the
generalization myth discovered by \cite{Zhang2016b} has been solved by
\cite{Soudry2017} and \cite{Poggio2018}, a neat generalization error (GE) bound
like the one for Support Vector Machine has not been worked out. The insights
obtained to control the error of adversarial samples will help us formulate a
lean bound. 3) When the optimization and generalization behaviors are worked
out, it would make it possible to formulate quantitative criteria concerning
the training set performance and test set generalization solely based on
network parameters. The criteria would make the training of dynamically-sized,
self-generating networks possible. 4) Anticipating the coming of the
self-generating networks, and more importantly, for the long term goal
described below, we are designing a new NN library that can mix Lisp, C++ and
CUDA arbitrarily. The idea is that a difference exists in the mathematical
reasoning language (the front end), and the underlying hardware implementation
language (the back end). The front end is universal, since it is based
on mathematics. The back end is hardware-dependent, it could be GPU based for
now, yet other types of hardware are possible, e.g., Tensor Processing Unit of
Google. The back end idea has been around for a while, e.g., XLA in Tensorflow
\citep{Abadi2015}, or even earlier, intermediate representation in LLVM, but
our front end perhaps will be more than computational graph.

\paragraph{Long term goal: unification of logic and perception.} For the above
goals, we have a rough guess on how they could be achieved, thanks to the cracks
opened by S-System. However, our ultimate goal is to combine logic with
perception. What S-System achieves is an algorithm framework that builds a
representation of the physical world. In other words, it is the mathematics of
imagination. However, for now, it only characterizes rather basic hierarchical
composition relationship between the representation of objects. The next big
extension is to characterize how these representations are manipulated in a
more universal way, which we believe is the key to rational reasoning found in
human intelligence. That's one key reason why we are working on a new library, since we
believe the days of Lisp may come again when we reach this part of
the road map. {\it This is the reason why we name the framework S-System: we
envision one day S-Expression \citep{McCarthy1960} and S-System would become
dual views of a one.}

\paragraph{Temporary end goal: mathematics of languages.} Languages are
probabilistic logic rules to manipulate symbols that are representations of the
physical world. The idea probably dates from Wittgenstein, who said in each
sentence, there is a picture. The idea probably was taken very seriously in the
Vienna Circle \citep{sigmund2017exact} in logic empiricism/positivism. But they
missed an important piece from Phenomenology \citep{dreyfus2015computers},
which is what machine learning was set out to fill. The thread is rather
simplified, and we will do a better survey when we finally reach the end of the
road map. The unification of logic and perception would make possible the
mathematics of languages, and consequently making true understanding of books
possible. {\it It would lay the theoretical foundation to build a system that
would organize all human knowledge that serves as an infrastructure may help
humanities get out of the post-modern quandary.} More on this in
\cref{sec:artif-intell-human}.

\subsection{Why A Theory of Neural Networks?}
\label{sec:why-theory-neural}

\subsubsection{The unsatisfactory progress of theory research and the problems
  it induces}
\label{sec:unsat-progr-theory}

The recent development in the algorithm family of neural networks (NN)
(\cite{LeCun2015}) that aim to solve high dimensional perception problems, has
led to results that sometimes outperform humans in particular datasets, e.g.,
vision (\cite{He}). It is a computational imitation of biological NNs
(\cite{Rosenblatt1958} \cite{Fukushima1980} \cite{DavidE.Ruineihart1985}).
Researchers of various backgrounds have been intrigued by theoretical
understanding of NNs, and made important contributions to it. From the
perspective of physics, we have \cite{Mehta2014} \cite{Lin2017}
\cite{Choromanska2015}; from that of applied mathematics, we have
\cite{Mallat2016}; from that of information theory, we have \cite{Amari1995}
\cite{Shwartz-Ziv2017}; from that of theoretical computer science, we have
\cite{Arora2013}; and from the machine learning perspective, we have
\cite{Anselmi2015} \cite{Anselmi2015a} \cite{Pennington2017a} \cite{Author2016}
etc.

However, despite the above efforts, it has arguably progressed for six decades
as a blackbox function approximator. It lacks a formal definition
\citep{Candes2002}, and has opaque optimization behaviors, mysterious
generalization errors \citep{Zhang2016b}, and intriguing adversarial samples
\citep{Szegedy2013}. The confusion of definition leads to the confusion of the
capability, and more importantly the limitation of the algorithm family, which
has led to criticism for not being able to undertake tasks that perhaps outside
the domain of NNs \citep{Pearl}, the worry of technology wall in the industry and
another AI winter, and ungrounded fear of singularity in the general public. The
lack of understanding of the behaviors of NN leads to a proliferation of competing
techniques and tricks that are hard to assess their relative merits
\citep{Lipton2018}. It is clear, for the field of deep learning to progress
healthily, a theory underpinning is being called besides its intellectual
reason.

\subsubsection{ Artificial intelligence and humanities}
\label{sec:artif-intell-human}

Zooming out from the academic world, intellectually, what is really significant
about the success of artificial NNs, and the development of its science? It is
perhaps the Homo Sapiens' unusual opportunity to fix its own errors, drag
itself out of the quandary of post-modernity, and build a future that
has a better dual (Yin Yang) balance of liberty and equality.

The general public may
be appalled, or perplexed by futurists' attention-grabbing prediction on the
technology singularity, in which the emerging of artificial superintelligence
may disrupt human civilization, and make humans look like bugs. But
intellectuals and scientists, we believe, are long at peace with the picture
painted by science. The shock of the picture had been long ago manifested by
the painting {\it Philosophy} (\cref{fig:philosophy}) of Gustav Klimt, when he
was commissioned to depict {\it The Triumph of Light over Darkness} to
``deliver an optimistic glorification of progress''. Klimt instead paints ``On
the left a group of figures, the beginning of life, fruition, decay. On the
right, the globe as mystery. Emerging below, a figure of light: knowledge.'',
which presents an a dreamlike mass of humanity, referring neither to optimism
nor rationalism, but to a "viscous void" \citep{klimit}. The tension depicted by
{\it philosophy} dates back to the onset of the Romantic age \citep{barzun2001dawn},
for science's abstruseness and abstract, and perhaps more importantly, for that
it shows man is not``the center of the universe'' in primitive mythologies, not
even ``the measure of all things'' in the Renaissance, but animals developed from
stardust in a trial and error fashion in perhaps a corner of the vast, violent,
lonely, impersonal universe. The tension gradually developed into a disillusion
with science \citep{watson2002modern}. The last straws are Godel's incomplete
theorem and phenomenology. The two combined shows that logic reasoning cannot
reach truth, all laws are human invention, and their correctness are entirely
based on perceptual empirical observations. The uncertainty of truth, and the
bust of the Utopian dream where people are allocated what they need gradually
morphed the society into post-modernity, where ``postmodernism resigns to the
alienation, ephemerality, fragmentation and patent chaos of modern life and
places individualized aesthetics over science, rationality, politics and
morality'' \citep{harvey1989condition}.

\begin{figure}[h]
  \centering
  \includegraphics[width=0.4\textwidth]{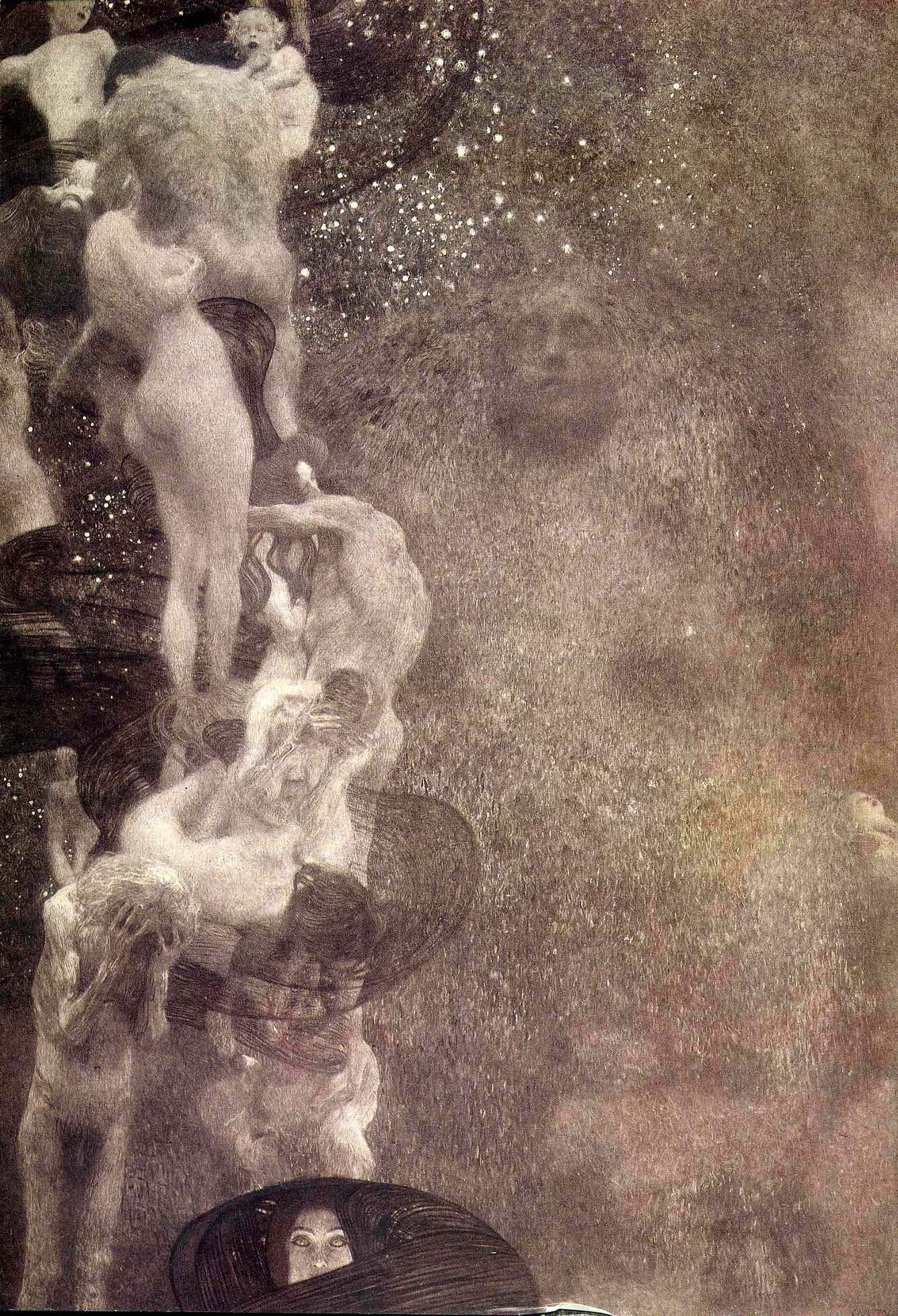}
  \caption{The painting {\it Philosophy} by Klimt.}
  \label{fig:philosophy}
\end{figure}

In the masterpiece of \cite{barzun2001dawn}, {\it From Dawn To Decadency}, in
stating the reason why he judged the Western culture is in decadency, he wrote
\begin{quotation}
  But why should the story come to an end? It doesn't, of course, in the
literal sense of stoppage or total ruin. All that is meant by Decadence is ``falling
off.'' It implies in those who live in such a time no loss of energy or talent or
moral sense. On the contrary, it is a very active time, full of deep concerns,
but peculiarly restless, for it sees no clear lines of advance. {\it The loss it faces is
that of Possibility.} The forms of art as of life seem exhausted, the stages of
development have been run through. Institutions function painfully.
Repetition and frustration are the intolerable result. Boredom and fatigue are
great historical forces.
\end{quotation}

It is a quite proper assessment, given our societies' situations. It even may
give a sense of relief for expressing the subconscious public mood. Even Hacker
Spirits \citep{Himanen:2001:HES:558235} has lost its moment \citep{levy1984hackers}, and was tamed by
capital \citep{turner2010counterculture}.

Yet, this is not the end of the story. The progress in modern biology, complex
system, nonlinear science has drawn another humanitarian epic narrative
\citep{edward1999consilience}, out of the Darwin jungle one --- indeed we are
stardust, but we are also children of the sun: ``Human
social existence, unlike animal sociality, is based on the genetic propensity
to form long-term contracts that evolve by culture into moral precepts and
law. They evolved over tens or hundreds of millennia because they conferred
upon the genes prescribing their survival and the opportunity to be represented
in future generations. We are not errant children who occasionally sin by
disobeying instructions from outside our species. We are adults who have
discovered which covenants are necessary for survival, and we have accepted the
necessity of securing them by sacred oath.''. The consilience of knowledge has
the potential to build a shared culture for all \citep{edward1999consilience},
beyond nationalism, tribalism, region fundamentalism, racism
(\cite{castells2000rise} \cite{castells2011power} \cite{castells2010end}):``We
are gaining in our ability to identify options in the political economy most
likely to be ruinous. We have begun to probe the foundations of human nature,
revealing what people intrinsically most need, and why. We are entering a new
era of existentialism, not the old absurdist existentialism of Kierkegaard and
Sartre, giving complete autonomy to the individual, but the concept that only
unified learning, universally shared, makes accurate foresight and wise choice
possible.\citep{edward1999consilience}''

However, the message is guarded in a high castle. To understand and act upon it
requires tremendous efforts. ``To which what one Crates’ said of the writings
of Heraclitus falls pat enough, 'that they required a reader who could swim
well,' so that the depth and weight of his learning might not overwhelm and
stifle him.''\citep{montaigne2004complete} If the enlightenment project of
consilience is to be successful, it has to be inclusive to all that are willing
to take the leap: the ones sense that something deep
in the contemporary culture is wrong, but do not have a thread or the knowledge to identify the
exact source, then initialize the work to fix it. We {\it lack the infrastructure} to give the opportunities to those
who want to be part of this Icarian flight: the bridge between the real
problems and the skills that are conducive to a solution, instead of a cannily
maintained tribalism elite culture that insecurely bars the door and
concentrates all resources \citep{deresiewicz2014excellent}. Without it, the
society in an optimistic guess may be gravitated towards self-perpetuating
tech-meritocracy parishes, in which the masses are lived by a certain
implementation of universal basic income without a meaningful goal for life as
depicted in the drama {\it The Expanse} \citep{expanse}. Or even worse, the
un-channeled negative emotion may build up, then explode to consume the earth
through radical populism, or radical religion fundamentalism, as having
happened repeatedly in the history. We believe the empowering is the Promethean
fire that the enlightenment project was set out to bring in the beginning: the
age that symbolically started by Francis Beacon.

There are reasons to believe that we are at the night before the dawn of a new
age, as the Renaissance is built on the ashes of the corrupt monastery order
\citep{mee2016life}, on the condition that we fix our errors --- we already
have a direction where our new philosophy may come from, thanks to forerunners
like \cite{edward1999consilience}. We have incrementally managed to build many
new experimental infrastructures. To name an incomplete list: Linux, an
operating system that is owned by all; Wikipedia, the open and free
encyclopedia accessible to all; search engines to access to the world's
knowledge freely and instantly as long as it is online; 3D printing to bring
the software power physical without heavy capitals to build a factory; RISC-V,
an open source hardware counterpart to Linux; digital libraries that open
research results to all, e.g., archive.org, and arxiv.org, where this paper is
submitted; interplanetary file system \citep{Benet}, and the distributed
Internet built on it; the possibilities of new content based business models
enabled by blockchain based currency \citep{Nakamoto2008} instead of the ad
based one that feeds on traffic and appeases to baser instincts.

In this particular niche domain of artificial NNs, it is not about a gold mine
for capitalistic opportunities, a boost in the efficiency of the current system,
potential dangerous superintelligence or ``smart'' cities. It is a thread to
decode how humans learn, so that if the science of learning is worked out,
imagine what that would bring? The enlightenment was initiated by the printing
press, which enables the massive dissemination of knowledge and information. In
the contemporary society, the situation is reversed. It is not the access to
information that becomes the bottleneck, but the capacity to process it: we are
drowning in information, but starving for knowledge. The internet has pushed the
situation to the extreme, and helped build polarized societies --- Google is
powerful on the condition that you ask the right question. If the science of
learning is worked out, we may build an upgraded version of Google Books, or in
other words, finish what Google Books set out to be: a system that organizes
all human knowledge, and serves as a guide to everyone who is willing to learn
and tackle challenges humanity facing, thus offers world-class education to
everyone, and continues the enlightenment project initiated by the printing
press, while fixing its errors.

Artificial Intelligence (AI) is part of the pieces mentioned above to build the
infrastructures that help build solve the tension between liberty and equality,
or laissez-faire and social welfare, by providing equal educational
opportunities, at least offering a low cost thread to navigate the ocean of
knowledge humanities has generated. Though how unconvincing the following
statement may be, AI could not be a threat. All tools have two edges. If we are
equipped the wisdom to wield it, it would not be a threat, and is fundamentally
controllable, at least the part based on the mathematics in this paper, which
covers deep NNs, the backbone of existing deep learning technology.

As a concluding remark, this section was initially planned as a short paragraph
explaining the deeper motivation for this paper that has been there for
years. It unexpectedly developed into a full section. The content is surely underdeveloped, and may
seem naive in the years to come. {\it This is the
first time the idea is on paper, feedback and help are welcome, by sending them
directly to the author, the email of which can be found at the author
information in the first page. The help part is described in details in
\cref{sec:call-help} below, and the {\bf footnote} of this page.}

\subsection{Call for Help}
\label{sec:call-help}

{\bf We have some trouble identifying appropriate journals, or conference to
publish this manuscript, in whole, or in parts.} The theory has drawn
techniques and ideas from a range of domains, and is rather
interdisciplinary. The author is a self-taught outsider who navigates the
domains involved all by himself. We are not deeply familiar with the
established manner of communication in the domains involved, and
trial-and-error submission would take too much time, which could be better used
in other areas. {\bf We very much appreciate experts in relevant domains to be
our guide.}\footnote{The author is willing to take visiting scholarship, or
research position to work on the publication of this manuscript. In addition,
the author is looking for a strong community to get a PhD that is working on the fundamental
theory of NNs, and is interested in the roadmap present in
\cref{sec:roadmap}. The author has been laboring in the darkness
to get this preliminary theory out in the past six years, and is in great need
of a community. If that is not an option, he is also seeking employment, for he
is not just a researcher, but also an engineer in the first
place --- the theory partly started with the hacker idea to build a AI, one needs to
understand it. For details, refer to
\href{http://shawnleezx.github.io/employment}{http://shawnleezx.github.io/employment}} {\bf In the
following, we will debrief the domains involved, and where we want guide.}

Expertise is needed in introducing {\it random matrix} techniques to the
machine learning community. We may manage to get \cref{sec:optim-landsc-s}
published on its own, since it solves a well-perceived open optimization
problem of NNs --- it does not mean that the rest does not solve open problems; it is
just that the problems are less perceived. What's left to be done is to write an
easy-to-follow introduction.

Expertise is needed in the boundary between {\it pure mathematics} and {\it
applied mathematics}, between {\it probability} and {\it statistics}. The
trickiest part is S-System in \cref{sec:coupling}.  Coupling theory is
traditionally a proof technique, here it is used as a computation technique. It
is a new way to assign probabilities, with deep philosophical implications (which
is not discussed in this paper), besides the symmetry principle and the law of
large number. So this makes it pure mathematics. However, it is also an
algorithm for concrete machine learning problems, which makes it applied
mathematics. It is a marriage between probability measure and statistics in a
way other than statistical learning theory (STL), or in other words, it is a
missing piece of the theory of {\it algorithmic modeling in statistics}
\citep{Donoho2000} \citep{Breiman2001}, along with STL. We are not familiar
with relevant venues, and are not sure where we should submit it: the feedback
from the machine learning community says the idea is abstruse, while the math
community seems not to have a clear community on the problem.

Expertise is needed in {\it information geometry} and {\it probabilistic
graphical modeling}. \Cref{sec:expon-family-neur} and
\cref{sec:learn-fram-htms} cannot be understood without \cref{sec:coupling},
though they also involve different communities. The former on information
geometry. The later perhaps is less clear --- it belongs to a community that it
sets out to prove ill-directed, i.e., probabilistic graphical modeling approach
in machine learning; the task is obviously difficult. We need to get the ideas
of \cref{sec:coupling} through before expanding on these two parts. To publish
them on their own, one approach is to demonstrate its usefulness in strong
experiment results without requiring the reviewers to understand the
nitty-gritty of the technical contents, but it would take some time to identify
the proper practical problem to solve.

Perhaps the best way is to publish it in somewhere like the {\it Philosophical
transactions. Series A, Mathematical, physical, and engineering sciences},
where a paper \citep{Mallat2016} similar in style published. But it would be
even harder than the previous routes perhaps: not only expertise is
needed, but also reputation; \cite{Mallat2016} also published in piecemeal first, and
took four years to reach the point to write a summary.

We are doing this in hackers' style: we work out an initial solution,
and present it to the community, like the situation of GNU/Linux in the
beginning. We hope the community would take notice, and help work on this in
joint efforts.

\subsection{Acknowledgement}
\label{sec:acknowledgement}

The work would be indefinitely delayed, or even impossible without the support
and tolerance of my past mentors. In reverse chronological order, I would like
to thanks Kui Jia at the South China University of Technology, without whom to
offer an extended stay in the academic world with a high degree of research
freedom when I graduated my Master of Philosophy degree and decided to return
to mainland China, I am not sure how long it would take, or even possible to
reach this paper if I went off the industry; Xiaogang Wang, who was my master
degree mentor at Chinese University of Hong Kong, and along with Kui to show me
the works of Stephane Mallat \citep{Mallat2012a} to introduce me to the theory
research in NNs, and helped me navigate the academic community; Jinhui Yuan, my
mentor at Microsoft Research Asia, to show me a broader world of machine
learning, and share his enthusiasm to research; Linli Xu, who was along the
first ones to introduce me to the field of machine learning, and be my
undergraduate thesis advisor at University of Science and Technology of China ---
it was around that time that I started to be cognizant of and ponder on the
theoretical problems of artificial intelligence, machine learning, and deep
neural networks.

I am also equally in great debt to past great men. Without an imagined community
built by them, and their stories to keep me company in the dark times, the
mental wear and tear would have me give up on the project a long time
ago. Without a particular order, they are Bertrand Russell, who spent ten years
on {\it Principles of Mathematics} to search for truth in logic, and had to
cover some of its publication expanse on its own after those ten years' work;
Nicola Telsa, who worked in a ditch to find the next meal when working on his
three-phase electric power, and who used up all his life force on his worldwide
wireless communication system; Albert Einstein, who spent seven years as a
patent officer when working on his theory of relativity, light quanta, and
Brownian motion; Voh Gogh, the father of modernism art who was so afflicted by
modernity, and battled with it on the edge of insanity; Desiderius Erasmus who
worked on the Latin and Greek translation of New Testament quite literally
against the world (for religion reform) in chronic poverty; Montaigne who
stayed a sane philosopher to write {\it Essays} in a world ripped apart by
religion factions. Yangming Wang, a great sage who cannot really be
characterized in one sentence without sufficient culture background. The list
can go on, but I think I have sent the message. If not, to illustrate what I
mean, Russell wrote in his autobiography:
\begin{quotation}
  At the time, I often wondered whether I should ever come out at the other end
of the tunnel in which I seemed to be.... I used to stand on the footbridge at
Kennington, near Oxford, watching the trains go by, and determining that
tomorrow I would place myself under one of them. But when the morrow came I
always found myself hoping that perhaps “Principia Mathematica” would be
finished some day.
\end{quotation}

\subsection{Notation}
\label{sec:notation}

 We note the notation used. All
scalar functions are denoted as normal letters, e.g., $f$; bold,
lowercase letters denote vectors, e.g., $\bvec{x}$; bold, uppercase letters
denote matrices, e.g., $\bvec{W}$; normal, uppercase letters denotes random
elements/variables, all the remaining symbols are defined when needed, and should be self-clear
in the context. r.e. and r.e.s are short for random element, random elements
respectively, so are r.v. and r.v.s for random variable and random variables.
To index entries of a matrix, following \cite{Erdos2017}, we
denote $\bb{J} = [N] := \{1, \ldots, N\}$, the ordered pairs of indices by $\bb{I} =
\bb{J} \times \bb{J}$. For $\alpha \in \bb{I}, A \subseteq \bb{I}, i \in \bb{J}$, given a matrix
$\bv{W}$, $w_{\alpha}, \bv{W}_{A}, \bv{W}_{i:}, \bv{W}_{:i}$ denotes the entry at
$\alpha$, the vector consists of entries at $A$, the $i$th row, $i$th column of
$\bv{W}$ respectively. Given two matrix $\bv{A}, \bv{B}$, the curly inequality
$\preceq$ between matrices, i.e., $\bv{A} \preceq \bv{B}$, means $\bv{B} - \bv{A}$ is a
positive definite matrix. Similar statements apply between a matrix and a
vector, and a matrix and a scalar. $\succeq$ is defined similarly. $\kappa$ denotes
cumulant, whose definition and norms, e.g., $|||\kappa|||$, are reviewed at
\cref{sec:eigenv-distr-symm}. $:=$ is the ``define'' symbol, where $x := y$
defines a new symbol $x$ by equating it with $y$. $\tr$ denotes matrix
trace. $\dg(\bv{h})$ denotes the diagonal matrix whose diagonal is the vector
$\bv{h}$.

\section{Neural Network, A Powerful Inference Machine of Plausibility of Events
in the Physical World}
\label{sec:complex-system}

This section aims to give an intuitive description of the formal definition of
NN given and its behaviors without delving into mathematical details; and
it also serves as the paper outline. {\it The rest of the paper characterizes the informal
description in this section formally.}

To avoid philosophical debates, we start with a metaphysics assumption that an
objective physical world exists independent of the perception/observation of
any agents, yet we do not make any assumptions on the nature of the physical
world, being it a simulation, or not. If an agent wants to interact with the
world, for whatever reasons, it first needs to perceive it. A way to perceive
the world is to measure certain physical objects of the world, which could be
implemented as sensors of the agent, e.g., a camera measuring the spatial
configuration of photon intensity. However, the measurement data (formalized in
\cref{cm:1}) record many events that happen in the same spatial and temporal
span, and those events are entangled in the measurements (formalized in
\cref{a:containment}). Let's see an example. A core drive of a living
creature, is to survive. When a lion is within 100m of a dear, the dear needs
to run in order to avoid being eaten. But how a dear is supposed to know a lion
is near by? A scheme could be through a photon sensor, i.e., eyes, of the dear
to measure the photons, probably generated from the sun, that are reflected
from the body surface of the lion. We call a lion exists nearby an {\it
event}. However, a great number of other events are also being measured by the
sensor --- the dark clouds that may rain, the delicious grass that makes food,
and the hyenas that lurks around for a leftover meal. {\it To perceive events
happening in the world, e.g., a lion is nearby, a mechanism is needed to
recognize it from measurement data, e.g., sensed spatial photon patterns.}

This leads to the problem of the structure of the physical world, which the
still unknown mechanism at least needs to relate to if it aims to recognize
events in the world well. This leads us to Complex System
(\cite{Bar-Yam:1997:DCS:331688} \cite{Nicolis:2012:FCS:2331101}
\cite{Newman2009}). \expand{It is a vast and diverse field, but a summary would not be
too misleading is that it is a science that studies the interaction of so many
events that their collective behavior manifests at a scale beyond their
characteristic scale.} The research in the field allows us to safely say that one of
the most important structures of the physical world is hierarchy
(\cite{Amderson1972}). Atoms form molecules; molecules form organism, and
inorganic material; organisms form creatures, which form ecology system;
inorganic material forms planets, then solar systems, then galaxies. {\it Of
the scales across the hierarchy of the world, the events at a lower scale interacting with
a particular way form events at a higher scale.} The hierarchy in nature is
formulated as assumptions \ref{a:physics} \ref{assumption:hierarchy} \ref{a:resolution}
\ref{a:measurability} \ref{a:containment}.

In \cref{sec:coupling}, S-System is introduced in \cref{def:s-system} to
recognize and represent the hierarchy of events from
measurement data, formulated measure-theoretically and based on probability
coupling theory (\cite{CIS-9904}). S-System formalizes the idea that a creature is not going to
reproduce an impartial representation of the world, it only captures the events
that cater to its need, e.g., the survival need to identify lion, and within
its reach and capacity, i.e., the amount of measurements it can gather. {\it
S-System recognizes a hierarchy of events from the measurements, not exactly in
the sense of physical reality --- if a creature never measured/saw a black
swan, it does not mean there aren't any --- but in the sense of a manually
created hierarchical groups of events, and each group is perhaps given a
tag/name.}  For examples, some low level object groups are named as edges;
groups in a relative higher level are named textures; groups in another higher
level are named body parts; groups in an even higher one are named body, e.g.,
lion.

In \cref{sec:expon-family-neur}, a NN is shown to be an implementation of an
S-system (shown in \cref{def:mlp_cge}), when the measure is being approximated
with compositional exponential family distributions (defined in
\cref{def:c_exp}). The unique properties of exponential family distributions
enable us to define a geometry structure on the intermediate feature space of
NNs based on information geometry. The most important discovery (for now) in such a
structure is that it shows that the ``distance'' between events (intermediate
features) is contracted as layers stack up, e.g., the intra-class difference
between an object and a slight deformed object (c.f. \cref{thm:cg}). In more details, Activation
Function is derived formally in \cref{def:relu}. The geometry structure of the
representation built by NNs is defined in \cref{def:nn_manifold} as stochastic
manifolds.  The section shows that a NN is an inference system to infer how
plausible groups of events forming a hierarchy have occurred, as S-System is
defined.

In \cref{sec:learn-fram-htms}, the learning framework of S-System (NNs) is
introduced. It formalizes how past repeated observations are summarized to
approximate the probability of groups of events given the current observation:
through optimizing an objective function.

The practicality to optimize the objective function is addressed in
\cref{sec:optim-landsc-s}. The hierarchical organization of NNs to recognize and
represent the hierarchical
physical world has made itself a complex system. {\it Contrary to existing
shallow models, or simple models, it is exactly the complexity built in
hierarchy that makes NNs the most powerful inference model.}  \expand{In a complex system, when the collective behavior of
events is stable, in the sense that it emerges on the overall interaction of
lower scale events, yet does not depends on any small subsets of them, it is
termed an {\it emergent} behavior. For instance, the magnetic force of a magnet
is a macroscopic phenomenon, yet it may emerge from the alignment of spin
magnetic moment of many elementary particles. A misalignment/disorder in a fraction of the
particles that does not exceed a critical points to cause phase transition
won't destroy the magnetic field of the magnet, only weakening it. More
examples and more thorough and in depth discussion about emergence could be
found in \cite{Kornberger2003} \cite{Nicolis:2012:FCS:2331101}
\cite{kadanoff2000statistical}. Contrary to the high constrained interaction in
the spinning particles, the neuron population in a NN is highly flexible and
adaptable.} A large collection of cooperative yet autonomous neurons,
formalized as assumptions \ref{a:boundedness} \ref{a:diversity}, gives NNs
the ability to partition events into arbitrary groups and infer the
plausibility of any groups of events (proved in
\cref{thm:landscape}), which is the emergent behavior emerging from
the disorder in the NN complex system. \expand{ More technically, formulated as
an optimization problem to minimize the error between the empirical probability
distribution and its parametric representation, though being non-convex, NNs
can reach zero error as long as NNs maintain its diversity while increasing its
neuron population.}

The four sections \ref{sec:coupling} \ref{sec:expon-family-neur}
\ref{sec:learn-fram-htms} \ref{sec:optim-landsc-s} respectively deal with the
definition of a hypothesis space, the geometry of the space, learning
objectives of the space, and the optimization landscape of the objective
functions. They combined form a preliminary theory of NNs and S-System, a
principle that assigns probability to physical events through learning on the
past experience, to infer the plausibility of events might happen in the
future.

{\bf Warning.} We will give a theory drawing techniques and ideas from a
range of fields --- probability coupling theory \cite{CIS-9904}, statistic
physics \cite{kadanoff2000statistical}, information geometry
\cite{amari2016information}, non-convex optimization \cite{Jain2017}, matrix
calculus \cite{magnus2007matrix}, mean field theory
\cite{kadanoff2000statistical}, random matrix theory \cite{taotopics}, and
nonlinear operator equations \cite{Helton2007}. As a word of warning, based on
the initial feedback obtained, it is very hard to understand respective parts
of the paper without having solid background in the above fields. We suggest
readers do prepare for a hard read.

\section{S-System: Physics, Conditional Grouping Extension and S-System}
\label{sec:coupling}

In this section, a mechanism, Hierarchical Measure Group and Approximate
System, nicknamed S-System, to
recognize and represent events in the physical world from measurements is
introduced with formalism from probability coupling theory (\cite{CIS-9904}) in a
measure-theoretical way.

\subsection{Physical Probability Measure Space and Sensor}
\label{sec:phys-prob-meas}

To begin with, we formally define the assumptions made on the physical world.
\begin{assumption}[Physics]
  \label{a:physics}
  All events in the physical world makes a probability measure space
$\mathcal{W} := (\Omega, \mathcal{F}^{\ca{W}}, \mu^{\ca{W}})$, where $\Omega$ denotes the event space,
$\mathcal{F}^{\ca{W}}$ is the $\sigma$-algebra on $\Omega$; $\mu^{\ca{W}}$ is the
probability measure on $\Omega$. We call $\ca{W}$ {\bf Physical Probability
  Measure Space (PPMS)}. To avoid confusion, we note that $\Omega$ denotes the event
space of PPMS throughout the paper.
\end{assumption}
\expand{
\begin{remark}
  For an example of the events, refer to \cref{sec:complex-system}; given $\omega,\omega'
  \in \Omega$ the intersection $\omega \cap \omega'$, union $\omega \cup \omega'$ and complement operations $\neg\omega$
  in $\mathcal{F}$ respectively denotes both $\omega, \omega''$ occurs, either one of $\omega,
  \omega''$ occurs, and $\omega$ does not occur. We shy away from giving an assumption of
  the meaning $\mu$, and merely rely on the axiomization by Kolmogorov, for it
is philosophical and metaphysical for now, though a reasonable
interpretation of the measure approximated by a NN is given later.
\end{remark}
}
\begin{assumption}[Hierarchy]
\label{assumption:hierarchy}
$\Omega$ has a hierarchical structure, which means $\Omega = \cup_{s \in
\mathcal{S}}\Omega_{s}$, where $s$ is named {\bf scaled parameter},
$\mathcal{S}$ is a {\bf poset}, i.e., a set with a partial order, $\Omega_{s}$ are
event spaces, and for $s, s''''' \in \mathcal{S}, s < s'$, $\Omega_{s'} \subseteq
\sigma(\Omega_{s})$, where $\sigma(\Omega_{s})$ is the $\sigma$-algebra
generated by $\Omega_{s}$.
\end{assumption}
  For $s, s' \in \mathcal{S}, s < s'$, we say $\Omega_{s'}$ is {\it composed} by $\Omega_{s}$. Furthermore, when $\omega_{s'''} \in
  \sigma(\cup_{s \in I_s}\{\omega_{s}\})$, where $I_s$ is an index set and for
  any $s \in I_s$, $\omega_{s} \in \Omega_{s}$,
we say $\omega_{s'}$ is composed by $\omega_{s}$.  As motivated in
\cref{sec:complex-system}, to perceive the events happening in the world,
measurements need to be collected, which is formalized as a r.e..
\begin{definition}[Measurement Collection]
  \label{cm:1}
  A measurement collection is a random function $X$ that supported on PPMS
$\mathcal{W}$ with an induced probability measure space $\mathcal{X} :=
(\Omega^{\ca{X}}, \mathcal{E}^{\ca{X}}, \mu^{\ca{X}})$, where $\Omega^{\ca{X}} := \{x | x: \mathbb{U} \to \mathbb{V}\}$ and
$\mathbb{U}, \mathbb{V}$ are unspecified the domain and codomain.
\end{definition}
 We make the following assumptions on $X$. It characterizes the capability and
limitation of a sensor and the phenomenon that for an event $\omega_{s'}$ composed
by a lower scale event $\omega_{s}$, the time/place/support where $\omega_{s'}$ happens
contains that of $\omega_{s}$.
\begin{assumption}[Resolution]
  \label{a:resolution}
  For any measurement collections, a lower bound of scale parameter $s_0$ exists, such that $\forall s
  \in \ca{S}$, $s$ is comparable with $s_0$, and $\forall \omega_{s} \in \Omega, s < s_0, \mu^{\ca{X}}(X(\omega_s))
  = 0$. \expand{In other words, part of measure of $\Omega$ in singular w.r.t. to
    $\mu$ of $\ca{X}$.} We call {\bf $\Omega_{s_0}$ the events of the lowest measurable scale}.
\end{assumption}
\begin{assumption}[Measurability]
  \label{a:measurability}
   measurements are physical, i.e., $\ca{E}^{\ca{X}} \subseteq \sigma(\Omega_{s_0})$.
\end{assumption}
\begin{assumption}[Containment]
  \label{a:containment}
  Given any two comparable scale parameter $s, s', s < s'$, $\omega_s \in \Omega_{s}, \omega_{s'} \in \Omega_{s'}$, where $\omega_{s'}$ is
composed by $\omega_{s}$, we have $\supp (X(\omega_{s})) \subseteq \supp (X(\omega_{s'}))$, where
  $\supp$ denotes the support of $X(\omega) \in E$, i.e., the domain of $X(\omega)$ where
  $X(w) \ne 0$ (we assume the zero element is defined, and indicates nothing
  has been measured).
\end{assumption}
\expand{
\begin{example}
  Suppose the sensor is an image sensor. $E = \mathcal{L}^{2}(\mathbb{R}^{2})$,
the square integrable function space defined on $\mathbb{R}$, and the scale
parameter could be interpreted as the distance $||x - y||_{2}$ between two
coordinates $x, y \in \mathbb{R}^{2}$ for a image $f \in E$ obtained by the image
sensor. In this case, a patch of the image is a lower scale event, which is
spatially contained in the whole image.
\end{example}
}

The containment phenomenon is troublesome, along with the phenomenon that it is
possible for any events $\omega, \omega' \in \Omega$ to have overlapping support $\supp(X(\omega)),
\supp(X(\omega'))$, even they do not have any composition relationship. This is an
inherent problem of measuring: it has collapsed all the events across scales
and within the same scale in the same measurement units of the sensor,
e.g., pixels at the same location in the image sensor. To perceive certain event
has occurred from $X$, a mechanism is needed to disentangle it from other events.

\subsection{S-System: Hierarchical Measure Group and Approximate System}
\label{sec:hier-meas-transp}

In this subsection, we introduce Hierarchical Measure Group and Approximate
System, nicknamed as S-System. Following the motivation described of S-System
in \cref{sec:complex-system}, \expand{a.k.a. to perceive is to select and group
events that serves certain needs of the creature but not to reconstruct
faithfully,} we create extensions of the probability measurable space
$\mathcal{X}$ that ``reproduce'' measure of higher scale events.  The
extensions are created hierarchically, by Conditional Grouping Extension
(CGE). For a review of the coupling theory and probability measure space
extension, please refer to \cref{sec:coupling-theory-1}.

\begin{definition}[Event Representation; (Partial) Conditional Grouping Extension]
  \label{def:cge}
  Let $T$ be a r.e. in measurable space $(E, \ca{E})$ defined on a
probability space $(F, \ca{F}, \mu)$, a {\bf Conditional Grouping Extension
(CGE)} of $T$ is created as the following by conditioning extension and splitting
extension.

First, a conditioning extension $(\Omega^{\hat{\ca{H}}^{e}},
\hat{\ca{H}^{e}}, \mu^{\hat{\ca{H}}})$ of $T$ is created with $((E, \ca{E}),
(\Omega^{\hat{\ca{H}}}, \hat{\ca{H}}))$ probability kernel $\hat{Q}(\cdot, \cdot)$, of which
an external r.e. $\hat{H}$ in measurable space $(\Omega^{\hat{\ca{H}}},
\hat{\ca{H}})$ is created with law
\begin{displaymath} \mu^{\ca{\hat{H}}}(\hat{H} | T) = \hat{Q}(T, \hat{H})
\end{displaymath} Then a splitting extension $(\Omega^{\ca{H}^{e}}, \ca{H}^{e},
\mu^{\ca{H}})$ of $(\Omega^{\hat{\ca{H}}^{e}}, \hat{\ca{H}}^{e}, \mu^{\hat{\ca{H}}})$ is
created with a $((E, \ca{E}) \otimes (\Omega^{\ca{H}}, \ca{H}), (\Omega^{\hat{H}},
\hat{\ca{H}}))$ probability kernel $Q(\cdot, \cdot)$ to support an external random
element $H$ in measurable space $(\Omega^{\ca{H}}, \ca{H})$ with law
$\nu$, of which
\begin{displaymath}
  \mu^{\ca{H}}(\hat{H} | H, T) = Q((T, H), \hat{H})\text{, and }\mu^{\hat{\ca{H}}}(\hat{H} | T) = \int Q((T,
  H), \hat{H})\nu(dH ; T)
\end{displaymath}
We assume that $\hat{Q}$ is a kernel parameterized by $W(T; \bv{\theta})$, a transport map
$W$ applied on $T$ parameterized by $\bv{\theta}$. The extension is well defined due
to \cite{CIS-9904} Theorem 5.1.

Let $\ca{M} := ((\Omega^{\ca{H}^{e}}, \ca{H}^{e}, \mu^{\ca{H}}), \{H, \hat{H}, T\})$, we
call $\ca{M}$ the event representation built on $T$ through a CGE ---
we define formally an {\bf event representation} is a pair, of which the first element is a
probability measure space, and the second element is a set of r.e.s
supported on the space, called {\bf random element set} of $\ca{M}$. When
absence of confusion, we just call $\ca{M}$ the event representation built on $T$.
$T$ is called the {\bf input random element} of $\ca{M}$; $H$ the {\bf group
indicator random element}; $\hat{H}$ the {\bf coupled random element};
$(H, \hat{H})$ the {\bf output random elements} when we would like to refer to
them in bunk; $\hat{Q}$ the {\bf coupling probability kernel}; $Q$ the {\bf
group coupling probability kernel}; $(\Omega^{\ca{H}^{e}}, \ca{H}^{e}, \mu^{\ca{H}})$
the {\bf coupled probability measure space}; $\mu^{\ca{H}}$ the {\bf coupled
probability measure}; $\nu$ the {\bf conditional group indicator measure}; $W(T;
\bv{\theta})$ the {\bf transport map} of $\ca{M}$. Given an $\omega \in \Omega^{\ca{H}}$, we
say $W^{-1}(\omega) \subset E$ is an {\bf event represented/indexed/grouped by} $H$.
Since CGE will be used recursively later, to emphasize, when $\ca{M}$ only
builds on a subset of output r.e.s of another event representation,
  to emphasize, $\ca{M}$ is called an event representation
built by a {\bf Partial CGE}.
\end{definition}
We explain why they are named as Conditional Grouping Extension and Event
Representation.  By assumption, $\hat{Q}$ is a probability kernel parameterized
by $W(T)$, e.g., the exponential family probability kernel $e^{\bv{w}^{T}\bv{x}
- \psi(\bv{w})}$, where $T := X, W(T) := \bv{w}^{T}X - \psi(\bv{w})$. Suppose
$X$, a measurement collection r.e., is supported by PPMS $\ca{W}$, from \cref{cm:1}. A
transport map $W$ applied on $X$ is a deterministic coupling $(X, W(X))$ that
transports the measure $\mu(A)$ of an event $A \in \ca{E}^{\ca{X}}$ to $W(A)$,
of which $W(X)$ is a r.e. on a measurable space $(\Omega^{\hat{\ca{H}}},
\hat{\ca{H}})$ with law $\mu^{\hat{\ca{H}}}$ supported on $\ca{W}$ where
\begin{displaymath}
  \mu^{\hat{\ca{H}}}(W(A)) = \mu^{\ca{W}}(X^{-1}(A)), A \in \ca{E}^{\ca{X}}, X^{-1}(A) \in \ca{F}^{\ca{W}}
\end{displaymath}
{\it That is to say $A$ is an event that are happening in the physical world,
and is being measured by $X$.} The goal of S-System is to estimate the
plausibility of the event $A$.  However, the problem is that we do not know
$\mu^{\ca{W}}$ (that's not to say we do not have an estimation of
$\mu^{\ca{W}}$ empirically). That's why CGE is needed. {\it CGE hypothesizes a
probability kernel $Q$ that approximates the probability $\mu(A)$ of events
being measured (conditioning extension) grouped by the r.e. $H$ created by
splitting extension.} Notice two key constructions to deal with two key
challenges here: for the enormity of the event space of PPMS, a.k.a. $\Omega$, only
events that happen along with current observation $X$ is estimated through
conditioning extension; for events happening along with $X$, {\it probability is
approximated in groups indexed by $H$} through splitting extension, which
physically could be broken down into countless smaller scale events that
compose $A$ and some of the sub-events won't be estimated. The design could be understood as
economic considerations, though probably it would be the only feasible solution
to reasonably approximate $\mu^{\ca{W}}$. Then, the r.e. set of $\ca{M}$ is the
manipulable object that directly connects with events in the physical world,
and is named event representation.

Yet, one more problem is looming around: how possibly $Q$ approximates
$\mu^{\ca{W}}(X^{-1}(A))$ reasonably? Suppose $A$ is a top scale event, by
\cref{assumption:hierarchy}, $A \in \sigma(\Omega_{s_L}) \subset
\sigma(\Omega_{s_{L-1}}) \subset \ldots \subset \sigma(\Omega_{s_0})$, where
$\ca{S}_{l} = \{s_l \}_{0\leq l \leq L, l \in \bb{N}}$ is a finite set of
scales. {\it Thus, to approximate $\mu^{\ca{W}}(W^{-1}(A))$ is to approximate the joint
distribution of events that compose $A$, which could be factorized into the
probabilities of events that compose $A$ and the probability of $A$ conditioning on
the sub-events}. This asks to apply CGE recursively, through which we get an S-system.

\begin{definition}[Hierarchical Measure Group and Approximate System]
  \label{def:s-system}
  A Hierarchical Measure Group and Approximate System (S-System) is a mechanism to extend the
  probability measure space of a measurement collection r.e. recursively
  according to a poset structure $\ca{S}^{\ca{X}}$ as described in \cref{alm:HMGAS}. The poset is
  called the {\bf scale poset} of the S-system. Ultimately,
  it creates an event representation $\ca{M}^{\bar{\ca{W}}} := (\bar{\ca{W}},
  \ca{O}^{\bar{\ca{W}}})$, where $\bar{\ca{W}} := (\bar{\Omega},
\ca{F}^{\bar{\ca{W}}}, \mu^{\bar{\ca{W}}})$ is the extended probability measure
space built and is called {\bf Approximated Probability Measure Space
  (APMS)}, and $\ca{O}^{\bar{\ca{W}}}$ is a r.e. set indexed by
elements of poset $\ca{S}^{\ca{X}}$. $\ca{M}^{\bar{\ca{W}}}$ is called the event representation built by S-System.
\end{definition}
\begin{algorithm}
  \caption{S-System. In the algorithm below, the \texttt{predecessor}(s) returns a index set $\bb{I}'$ that indexes
    elements in $\ca{S}^{\ca{X}}$ and $\forall i \in \bb{I}', s_i \leq s$ and
    \texttt{successor}(s) return a subset $\ca{S}'$ of $\ca{S}^{\ca{X}}$, where
    $\forall s_i \in \ca{S}', s_i \geq s$. For examples of the functions, refer to \cref{alm:func}.}
  \label{alm:HMGAS}
\begin{algorithmic}
  \Require $\ca{S}^{\ca{X}} := \{s_i\}_{i \in \bb{I}}$ is a poset with a
  minimal element $s_0$, whose elements are indexed by a set $\bb{I}$; $X$ is a
  measurement collection r.e. supported on $\ca{X} := (\Omega^{\ca{X}}, \ca{E}^{\ca{X}},
\mu^{\ca{X}})$, of which the events of the lowest measurable scale are $\Omega_{s_0}$
  \Ensure an event representation $\ca{M}^{\bar{\ca{W}}}$
  \State $\ca{M}_{s_0} \gets ((\Omega^{\ca{H}_{s_0}}, \ca{H}_{s_0}, \mu^{\ca{H}_{s_0}}) :=
  \ca{X}, \ca{O}_{s_0} := (X)), \ca{T}_{\text{out}} \gets \emptyset$,  $\ca{S}^{\ca{X}_t} \gets$ \Call{successor}{$s_0$}
  \While{$\ca{S}^{\ca{X}_t}$ is not empty}
    \State $\ca{S}^{\ca{X}_{t'}} \gets \emptyset$
    \For{$s \in \ca{S}^{\ca{X}_t}$}
      \State $\bb{I}' \gets$ \Call{predecessor}{s}
      \State $\ca{H} \gets \bigotimes_{i \in \bb{I}'}(\Omega^{\ca{H}_{s_i}},
      \mathcal{H}_{s_i}, \mu^{\ca{H}_{s_i}})$, $\ca{O} \gets \cup_{i\in \bb{I}'} \ca{O}_i$, $\ca{M} \gets (\ca{H}, \ca{O})$
      \State Build an event representation $\ca{M}_s := ((\Omega^{\ca{H}_{s}},
      \ca{H}_{s}, \mu^{\ca{H}_{s}}), \ca{O} \cup \{H_s, \hat{H}_s\})$ on
      $\ca{O}$ through conditional grouping extension that supports output
      r.e.s $(H_s, \hat{H}_{s})$
      \If {\Call{successor}{s} is empty}
        \State $\ca{T}_{\text{out}} \gets \ca{T}_{\text{out}} \cup \{\ca{M}_s\}$
      \Else
        \State $\ca{S}^{\ca{X}_{t'}} \gets \ca{S}^{\ca{X}_{t'}} \cup $\Call{successor}{$s$}
      \EndIf
    \EndFor
    \State $\ca{S}^{\ca{X}_{t}} \gets \ca{S}^{\ca{X}_{t'}} $
  \EndWhile
  \State  $\ca{M}^{\bar{\ca{W}}} \gets (\bigotimes_{i \in \bb{I}''}(\Omega^{\ca{H}_{s_i}},
  \mathcal{H}_{s_i}, \mu^{\ca{H}_{s_i}}), \cup_{i\in\bb{I}''}\ca{O}_i)$,
  where $\bb{I}''$ indexes all event representations now in the set $\ca{T}_{\text{out}}$
\end{algorithmic}

\end{algorithm}

\begin{algorithm}
  \caption{Example implementations of S-System functions {\tt successor} and {\tt predecessor}.}
  \label{alm:HMGAS_func}
\begin{algorithmic}
  \Function{successor}{$s$}
    \State
    \Return the set of elements in $\ca{S}^{\ca{X}}$ that are the immediate successors of
      $s$ (immediate successors of $s$ is the set of smallest elements that are
      comparable with $s$ and larger than $s$, though themselves are not comparable)
  \EndFunction
  \Function{predecessor}{$s$}
    \State
    \Return the set of indices of elements in $\ca{S}^{\ca{X}}$
    that are the immediate predecessors of $s$ (immediate predecessors of
    $s$ is the set of smallest elements that are comparable with $s$ and
    smaller than $s$, though themselves are not comparable)
  \EndFunction
\end{algorithmic}
\label{alm:func}
\end{algorithm}

\subsection{Related works}
\label{sec:related-works-1}

\subsubsection{Hierarchy}
\label{sec:hierarchy}

The idea that the data space that NNs process is hierarchically structured and
NNs are only operating in a rather small subset of the space, has been more or
less a folklore by the researchers in the neural network community. However,
the wide recognition of hierarchy has come late, mostly because the seminal work by
\cite{Krizhevsky2012} that proves the significance of hierarchy in NNs experimentally. The hierarchy
is mostly motivated by the imitation of biological neural networks (\cite{Fukushima1980}
\cite{Riesenhuber1999} \cite{Riesenhuber2000}), where neuroscience shows that it
has a hierarchical organization (\cite{Kruger2013}), and does not make the
connection to the hierarchy in nature, which is reasonable since at the time
NNs/Perceptron (\cite{Rosenblatt1958}) was invented, the Complex System
(\cite{Kornberger2003} \cite{Amderson1972}) that studies the hierarchy in nature
did not exist yet. The connection between hierarchy in nature and NNs has been
discussed qualitatively by physicists (\cite{Lin2017} \cite{Mehta2014}), though to
the best of our knowledge, a fully measure-theoretical characterization of the
hierarchy in the data space, described in \cref{sec:phys-prob-meas} does not
exist before. It gives a theoretical motivation of a hierarchically built
hypothesis space, i.e., S-System, contrary to the motivation of artificial NNs,
which is an imitation.

\subsubsection{Hierarchical Hypothesis Space of NNs}
\label{sec:nn-hypothesis-space}

Many works have been studying the hierarchical structure of the hypothesis space
of NNs. Though perhaps surprisingly, an informal idea similar with S-System has been underlying the
design of CNN (\cite{Lecun1998}) at the beginning, where in the unpublished
report \cite{Bottou1996}, they describe that it is better to defer
normalization as much as possible since it ``delimiting a priori the set of
outcomes'', and pass scores as unnormalized log-probabilities. However, perhaps
due to a lack of rigor, they removed the discussion in the formal
publication. The passing of scores corresponds to the deterministic coupling
that transports true measure in the PPMS, while normalization corresponds to
assuming a probability kernel to approximate the true measure
transported.

Further analysis on the hierarchical behavior of NNs waited for two
decades. Early pioneers analyzes from the perspective of kernel space and
harmonics. At the end of the dominant era of support vector machine (SVM),
\cite{Smale2009} seeks to give NNs a theoretical foundation in Reproducible
Kernel Hilbert Space (RKHS) (\cite{Vapnik1999} \cite{Scholkopf:2001:LKS:559923}),
which is an analogy but may only give limited insights. We will discuss
how RKHS relates to S-System later when we discuss the difference
between S-System and RKHS based nonlinear algorithms. Many works in this direction
have been done, either taking NNs as a recursively built RKHS
(\cite{Daniely2016a}), or applying the recursion idea to existing kernel methods
(\cite{Mairal2014a}). We do not aim to cover all kernel works. We envision it as a tool to aid
analysis, and design probability kernels in S-System, yet not as the fundamental
underpinning. A work (\cite{Anselmi2015a}) in the line of RKHS has
also sought foundation in probability measure theory, though its focus is the invariance and
selectivity of the one layer representation built by NNs. It studies the measure
transport due to compact group transformations, and points out that the output
of the activation function of NNs could be the probability distribution of low
dimensional projection of the measure of data and its transformations, which is
similar to the case where S-System only couples group indicator r.e. --- they
both analyze the grouping of measure transported by transport maps --- though
when taking on the hierarchical behavior, it falls back to RKHS, and think
recursion as ``distributions on distributions'' instead of coarse grained
probability coupling. We believe the work could be inspirational to further
refined analysis on r.e.s created by S-System. Under the umbrella of computational
harmonics, \cite{Mallat2012a} \cite{Mallat2016} understand NNs as a technique
that learns a low dimensional function $\Phi(x), x \in \ca{X}$ that
linearizes the function $f(x)$ to approximate on complex hierarchical symmetry
groups from a high dimensional domain $\ca{X}$. It achieves this by progressively contracting
space volume and linearizing transformation that consists of groups of local
symmetries layer by layer. However, the group formalism used is an
analogy that only rigorously characterizes Scattering Network (\cite{Mallat2012a}), a
hierarchical hypothesis space simplified from NNs, and does not characterizes
NNs. The group formalism is referred as the ``mathematical ghost'' in
\cite{Mallat2016}. We believe these works are important to further incorporate
symmetry structure in nature in S-System in future works.

More recently, \cite{Author2016} interprets NNs in Probabilistic Graphical Model (PGM).
It takes activation as
log-probabilities that propagate in the net. As the description suggests, it
confuses the transported measure to be approximated, and the approximated
probability obtained by a probability kernel. Thus, it has to rely on the Gaussian
assumption to justify the interpretation, of which the mean serves as
templates, and the noise free assumption to justify ReLU activation function. Also,
the assumption makes it a generative model that has to make assumptions on the
data distribution, while an S-system is able to only make assumptions on how
measure is supposed to group. From the spline theory perspective,
\cite{Balestriero2018} understands NNs as a composition of max-affine spline
operators, which implies NNs construct a set of signal-dependent,
class-specific templates against which the signal is compared via an inner
product. From S-System point of view, it is an analysis on
the functional form of coupled r.e.s of an S-system that assumes compositional
exponential probability kernels and does maximal estimation on
group indicator r.e.s. It connects more with the function
approximation results, that takes ``signal-dependent'' as a fact to see what that
implies, than the goal of S-System, i.e., giving a theoretical formal
definition and interpretation to NNs. We think it may contribute to the refined analysis of
decision boundaries in S-System in the future. Analogizing with statistical mechanics, \cite{Trevisanutto}
takes the group indicator r.e.s. with binary values as gates, of which the
expectation will multiply with the coupled r.e.s. to decide how much the
``computation'' done should be passed on to next layers. However, what is being
computed is left unspecified. As in the definition of S-System, the computation
is to extend the probability measure space of the measurement collection r.e. that aims to
approximate probability measure of events in the event space of PPMS. The group
indicator r.e.s. is not a gate, but serves to group measure. It behaves like a
gate when its value is binary, yet underlying it serves to create further
coupling of grouped measure. Thus, the analog does not unveil the deeper
principles underlying, e.g., probability measure space extension and the probability
estimation/learning happening in S-System (refer to \cref{sec:coupling}
\cref{sec:learn-fram-htms}).

\subsubsection{Machine Learning Algorithm Paradigm}
\label{sec:mach-learn-parad}

We envision S-System as an attempt that tries to investigate a
measure-theoretical foundation of algorithmic modeling methods
(\cite{Breiman2001}) for designing machine learning algorithms. Now we can see
NNs as an implementation of S-System, which is a way to transport, group and
approximate probability measure. From S-System, we can see that we do not need
to make assumptions on the distribution of data to justify that our model is
probabilistic --- the randomness comes from the data source itself, and it is the
probability measure space that a model is manipulating, not the probability
values. Thus, we can break from statistics methods developed ever since Ronald
Fisher that has to make assumptions on data, and proceed from there. This
measure manipulation paradigm may be a promising candidate to the theoretical
issue facing high dimensional data analysis (\cite{Donoho2000}). Thus, we
discuss current major algorithm paradigms in machine learning/high dimensional
data analysis, i.e., Support Vector Machine with Kernels (SVMK) and
Probabilistic Graphical Model (PGM).

It is well known that SVMK can be analogized to a NN with one hidden layer. The
hypothesis $f$ of SVM can be expressed as a linear combination of inner product
between test samples $\bv{x}_i$ and support vectors $f(\cdot) = \sum_{i} \alpha_i k(\cdot,
\bv{x}_i)$, where $k$ is the kernel function, and $\alpha_i$ scalars. Writing
$f$ in the form of $f(X) = \sum_{i} \alpha_i k(X, x_i)$, it can be seen that
the hypothesis is actually a deterministic coupling, where $f$ is the transport
map. As happening in \cref{sec:learn-fram-htms}, the training of SVM is also
minimizing a surrogate risk between the true data probability measure and the
transported measure, though no probability distributions are ever
introduced. The probability kernels in S-System is replaced by a positive
semi-definite (PSD) kernel, whose output value is a real number indicating
something similar with the coupled probability measure of S-System. This observation may seem
surprising, however, it makes much sense when we notice the fact that
probability is just a function. SVM is a function approximation techniques
designed specifically for the case where the data are of high dimensional, yet
the number of samples available is small. To combat the curse of dimensionality, it
uses a PSD integral operator (\cite{Partial1994}) that maps the sample to a high
dimensional space, which can be taken as templates, and only approximates
measure that is in the vicinity of those templates and ignores the rest of the
space. The kernel can also be built hierarchically, which is discussed in
\cref{sec:nn-hypothesis-space}. For the time being, S-System does not contain SVMK
as a special case, while we envision by properly generalizing the probability
kernels in CGE, a large class of algorithms may include SVM.

As for PGM, it is a special case of S-System. As mentioned repeatedly throughout
the paper, S-System merely makes assumptions on how measure is supposed to group,
without making assumptions on the actual distribution of the data. The learning
framework of S-System described in \cref{sec:learn-fram-htms} is actually the same
as PGM when only considering the unsupervised case, where assumptions on data
distribution have to be made. Thus, S-System is a superset of algorithms including
PGM. The graph in PGM is actually a poset. However, the insight comes from
where they differ. Relying heavily on the assumptions on the distribution of
data, which is in reality unknown, it introduces large model biases, which
perhaps is the reason why
it alone cannot compete with NNs on complex high dimensional data. Furthermore,
S-System is naturally compatible with supervised labels, since hidden
variables/group indicator r.e.s map one-to-one to labels, which dictates how
measure should be grouped. This point is discussed more thoroughly in
\cref{sec:learn-fram-htms}, where supervised and unsupervised learning are
taken as dual perspectives on the same object.

\section{Geometry of NNs (and S-System), and Neural Networks From the First Principle}
\label{sec:expon-family-neur}

In the previous section, a mechanism S-System is introduced to transport, group
and approximate probability measures that are of interest. It focuses on
deriving a mechanism to recognize events through measurements from the first
principle.  In this section, we will show that Multiple Layer Perceptrons (MLP)
(\cite{DavidE.Ruineihart1985}) is an implementation of an S-system. {\it The derivation serves as a proof
of concept, and as an example of S-System}, though we note that all existing NN architectures, e.g., Residual
Network (\cite{He2015}), Convolutional Neural Network (\cite{Simard2003}),
Recurrent Neural Network (\cite{Hochreiter1997}), Deep Belief Network \citep{Hinton2006}
(\cite{Hinton2006}) could be derived by using different measurable spaces, posets,
probability kernels and successor, predecessor functions, along with manifold
possibilities of new architectures. In the derivation, we will see classical
activation functions emerging naturally. Then, we go further to endow geometry
on event representations by defining the proper manifold structure on S-System
using information geometry. It enables us to quantitatively prove the benefits
of hierarchy that MLPs implement coarse graining that contracts the variations in
the lower scale event spaces when creating higher scale event extensions, which
plays the same role as RG in physics.

\subsection{Theoretical Derivation of Activation Functions and MLPs}
\label{sec:theor-deriv-nn}

Let the CGE in \cref{def:s-system} be MLPCGE (\cref{def:mlp_cge}), the
$\tilde{\bv{t}}$ in MLPCGE be obtained by transport map ReLU (\cref{def:relu}),
and the scale poset $\ca{S}^{\ca{X}}$ be a chain, i.e., a poset where all
elements are comparable. By \cref{alm:HMGAS}, we would obtain a MLP. The
definitions are given in the following.

\begin{definition}[MLP Conditional Grouping Extension]
  \label{def:mlp_cge}
  An MLP Conditional Grouping Extension (MLPCGE) is a CGE with the
  following measurable space and parametric forms of probability kernels
  \begin{gather*}
    (E, \ca{E}) = (\bb{D}^{n}, \ca{D}^{n}) \bigotimes (\bb{R}^{n}, \bb{B}(\bb{R}^{n})),
    \nu(\bv{h}|\bv{t}) = e^{\bv{h}^{T}\bv{W}^T\tilde{\bv{t}}}/\sum_{\bv{h}} e^{\bv{h}^{T}\bv{W}^T\tilde{\bv{t}}}\\
    \hat{Q}(T, \hat{H})
    = \hat{q}(\bv{t}, \hat{\bv{h}})
    := e^{\bv{1}^{T}\hat{\bv{h}}(\bv{t}) }
    /(\int e^{\bv{1}^{T}\hat{\bv{h}}(\bv{t}) }
    d\mu(\bv{t}))
    =  e^{\bv{1}^{T}\bv{W}^{T}\tilde{\bv{t}} }
    /(\int e^{\bv{1}^{T}\bv{W}^{T}\tilde{\bv{t}} }d\tilde{\mu}(\tilde{\bv{t}}))\\
    Q((T, H), \hat{H})
    = q((\bv{t}, \bv{h}), \hat{\bv{h}})
    := e^{\bv{h}^{T}\hat{\bv{h}}(\bv{t}) }
    /(\int e^{\bv{h}^{T}\hat{\bv{h}}(\bv{t})}
    d\nu(\bv{h}|\bv{t})d\mu(\bv{t})
    = e^{\bv{h}^{T}\bv{W}^{T}\tilde{\bv{t}} }
    /(\int e^{\bv{h}^{T}\bv{W}^{T}\tilde{\bv{t}}}
    d\nu(\bv{h}| \bv{t})d\tilde{\mu}(\bv{\tilde{t}}))
  \end{gather*}
  where $\bb{D}^{n}$ is a $n$-dimensional discrete-valued field, i.e., $\{0, 1\}^{n}$ or $\{-1, 1\}^{n}$, $\ca{D}^{n}$
is the $\sigma$-algebra generated by $\bb{D}^{n}$, $\bv{W}$ is a matrix (in this case, the
transport map $W(T; \bv{\theta})$ is the matrix $\bv{W}$ and parameters
$\bv{\theta}$ are $\bv{W}$), $\bv{t}, \hat{\bv{h}}, \bv{h}$ are realizable values of
r.e.s $T, \hat{H}, H$, and $\tilde{\bv{t}}$ is obtained by
applying a yet unspecified transport map on $\bv{t}$ --- for now, it could be
just taken as the output of an identity mapping and other possible forms are
introduced when discussing activation functions --- and
$\tilde{\mu}(\tilde{\bv{t}})$ is the law on $\tilde{\bv{t}}$ induced by the law
$\mu(\bv{t})$ on the input r.e. of MLPCGE. The meaning of the rest of the symbols is
same with those in \cref{def:cge}.
\end{definition}

Note that it is not possible to compute $\hat{q}(\bv{t}, \hat{\bv{h}})$, for
$\mu(\bv{t})$ is unknown. However, we can compute $\nu$ faithfully! This is because $H$
is a manual creation/grouping instead of inherent events in
PPMS Here, with some further reasoning, we will have the marvelous trick done
by NNs, i.e., the Activation Function (AF). The key is only to build a full, or
partial CGE upon r.e.s created by a previous CGE, using an
estimated value of $H$. The deeper principles of the estimation are described in
\cref{sec:learn-fram-htms}, which is the maximization of expected data log
likelihood, and is part of the learning framework of S-System. When a full CGE
is created upon output r.e.s. of a previous CGE, $\bb{D}$ is $\{0, 1\}^{n}$, and the
estimation is done through expectation or maximum, we recover the currently
best performing activation function Swish (\cite{Ramachandran2017}) or ReLU
(\cite{Glorot2011}) respectively; when a partial CGE is created on the group
indicator r.e.s, the estimation is done through expectation, and $\bb{D}$ is
$\{0, 1\}$ or $\{1, -1\}$, we recover classical activation functions Sigmoid or
Tanh respectively.

We derive ReLU as an example. The group indicator r.e.s $H$ divides the
measure transported from the event space of input r.e. $T$ to the event space of
$\hat{H}$ into groups. Intuitively, if $H$ divides the measure into two groups indexed
by elements of $\bb{D}$, and we assume $1$ collects the measure corresponds to an event
collection while $0$ collects the complement of the event collection (meaning
the event collection does not occur), given a
realization $\bv{t}$ of $T$, to recognize higher scale events composed by lower
scale events represented by $H$, we would like to {\it estimate} what events
are present in $\bv{t}$, and create {\it further coupling} with another CGE on
the events that are present. Formally,

\begin{definition}[Rectified Linear Unit (ReLU)]
  \label{def:relu}
   Let $T := (H, \hat{H})$ be r.e.s created by a CGE, an estimation
   $\tilde{\bv{h}}$ of a realization of the r.e. $H$ is obtained by
   \begin{displaymath}
     \tilde{h}_i = \argmax_{h \in \bb{D}} \nu(h | \bv{t}) = e^{h\bv{W}_{:i}^{T} \tilde{\bv{t}}} / \sum_{h}e^{h\bv{W}_{:i}^{T} \tilde{\bv{t}}}
   \end{displaymath}
   A further coupling is created by MLPCGE upon $T$, of which $\tilde{\bv{t}}$ is
   the estimated realization of the r.v. obtained by applying a transport map $\text{ReLU}$ on $T$
   \begin{displaymath}
     \text{ReLU}(T) := H \odot \hat{H}
   \end{displaymath}
    where $\odot$ denotes Hadamard product. Operationally, $\text{ReLU}$ is a binary
    mask $\bv{h}$ applying on the outputs (preactivation) $\hat{\bv{h}}$ of the transport map
    $\bv{W}$. \expand{The $i$th dimension $h_i$ of $\bv{h}$ can be estimated separately
    because $h_i, h_j$ are conditionally independent conditioning on $\bv{t}$
(\cite{Hinton2006}).}
\end{definition}

As can be seen, AFs is not a well defined
object, which is actually a combination of operations from two stages of
computation.

\subsection{NN Manifold and Contraction Effect of Conditional Grouping Extension}
\label{sec:contr-effect-coarse}

In \Cref{sec:coupling}, motivated by the hierarchy assumption
\ref{assumption:hierarchy}, we designed S-System. Here, using MLP as an
example, and also a proof of concept, we quantitatively show the benefits of
hierarchical grouping done in S-System by showing that irrelevant/uninterested
variations in the lower scale events could be gradually dropped by repeatedly
applying CGE, characterized by ``shrinking distance'' between events.

To characterize the distance, we need a geometry structure on event representations. We
give an initial construction built on information geometry
(\cite{amari2016information}). For a review of manifold and information geometry,
please refer to \cref{sec:manifold} \ref{sec:information-geometry}. To begin with, we define
\begin{definition}[Compositional Exponential Family of Distributions]
  \label{def:c_exp}
  The form of {\bf compositional probability distribution of exponential
family} is given by the probability density function
  \begin{displaymath}
    p(\bv{x}, \bv{h}; \bv{\theta}) d\bv{x}d\bv{h}
    = e^{(k(\bv{h}, \bv{x}; \bv{\theta}) - \psi(\bv{\theta}))}d\mu(\bv{x}) d\nu(\bv{h}),
    k(\bv{h}, \bv{x}; \bv{\theta}) = \langle \bv{f}(\bv{\theta}; \bv{h}), \bv{g}(\bv{x}) \rangle
  \end{displaymath}
  where $\bv{x}$ is realizable values of a multivariate random variable, $k$
is a function called {\bf compositional kernel} that for a given $\bv{h}$, $k$ is
the inner product between certain vector function $\bv{g}(\bv{x})$, called {\bf
sufficient statistic}, (of
which the component functions are linearly independent) and certain vector
function $\bv{f}(\bv{\theta}; \bv{h})$, called {\bf composition function}, $\psi(\bv{\theta})$ is
the normalization factor, and $\mu, \nu$ ares the laws on r.v. $\bv{x}, \bv{h}$,
respectively.
\end{definition}
Conditioning on $\bv{h}$, $p(\bv{x} | \bv{h}; \bv{\theta}) =
e^{k(\bv{h}, \bv{x} ; \bv{\theta})}d\mu(\bv{x})$ is of the exponential family. Actually,
it is of Curved Exponential Family (\cite{Amari1995}). The
parametric form of kernel $Q$ of MLPCGE is of the compositional exponential
family, where $k(\bv{h}, \bv{x}; \bv{\theta}) = \langle \bv{h}^{T}\bv{W}, \bv{x\rangle}$,
$\bv{f}(\bv{\theta}; \bv{h}) = \bv{h}^{T}\bv{W}, \bv{g}(\bv{x}) = \bv{x}$.

\begin{definition}[Neural Network Manifolds]
  \label{def:nn_manifold}
  Let $\ca{M}$ be an event representation built on a measurement collection r.e. $X$
through an S-system. If probability kernels $Q, \hat{Q}$ of all CGE in the S-system are of the compositional
exponential family, of which the composition kernel is parameterized by the
CGE transport map, then the function space $M_s$ of measure
$\mu^{\ca{H}}_s(\hat{H}, T | H), s \in \ca{S}^{\ca{X}}$, where $\ca{S}^{\ca{X}}$ is
the scale poset of the S-system, is a Riemannian manifold with the
following properties:
\begin{itemize}
\item $M_s$ has a coordinate system $\bv{\eta}_s|_{\bv{h}} = (\eta_1, \ldots, \eta_n)$ that is the dual
affine coordinate system of an exponential family distribution,
where
  \begin{displaymath}
    \bv{\eta}_s|_{\bv{h}}
    := \nabla_{\bv{f}(\bv{\theta}; \bv{h})} \psi(\bv{\theta})
    = \bb{E}[\bv{t} | \bv{h}]
    = \int\bv{t}\,d \mu^{\ca{H}}_s(\hat{H}, T | H)
    = \int\bv{t}\,q((\bv{t}, \bv{h}), \hat{\bv{h}})d\mu(\bv{t})
  \end{displaymath}
  $\bv{\theta}$ is the parameters of the transport map $W(T; \bv{\theta})$,
  $\nabla_{\bv{f}(\bv{\theta}; \bv{h})}$ takes derivatives w.r.t. composition function of
  $k$ and $\bv{t}$ is realizations of $T$. We call the
  coordinates {\bf neuron coordinates}.
\item $M_s$ has a Riemannian metric derived by the second order Taylor
  expansion of the dual Bregman divergence defined by
  \begin{displaymath}
    D_{\psi^{*}}[\bv{\eta}_s'|_{\bv{h}}:\bv{\eta}_s|_{\bv{h}}] := \psi^{*}(\bv{\eta}_s'|_{\bv{h}}) - \psi^{*}(\bv{\eta}_s|_{\bv{h}}) - \nabla\psi^{*}(\bv{\eta}_s|_{\bv{h}})^{T}(\bv{\eta}_s'|_{\bv{h}} - \bv{\eta}_s|_{\bv{h}})
  \end{displaymath}
  where $\psi^{*}$ is the Legendre dual of $\psi$. We call the divergence defined
  {\bf neuron divergence}.
\end{itemize}
\end{definition}

In the above definition, the stochastic manifold is defined by conditioning on
group indicator r.e.s $H$. To appreciate the definition, let's return back to
MLP. Let $\ca{M}$ be the event representation built on a measurement collection r.e. $X$ by a
MLPCGE, i.e., the measure of output r.e.s being $\mu^{\ca{H}} =
e^{\bv{h}^{T}\bv{W}^{T}\bv{x} - \psi(\bv{W})}\mu(\bv{x})\nu(\bv{h}|\bv{x})$. When $\bv{h}$ is fixed, letting
$\bv{f}_0 = \bv{h}^{T}\bv{W}^{T}$ we have $\mu^{\ca{H}}(\bv{x}|\bv{h}) =
e^{\bv{f}_{0}^T\bv{x} - \psi(\bv{W})}\mu(\bv{x})$. It is known (\cite{Nielsen2009}) that the
expectation statistics, i.e., $\bv{\eta}|_{\bv{h}} =
\nabla_{\bv{f}_0}\psi(\bv{f}_0(\bv{W}))$, uniquely determines
$\mu^{\ca{H}}(\bv{x}| \bv{h})$. It implies that given $\bv{h}$,
$\mu^{\ca{H}}$ is a probability distribution, of which the most ``salient''
feature is the expectation. This explains why the visualization of NN
representations tends to be templates (\cite{Mahendran2014} \cite{Zhang2018}),
and the template based theories (\cite{Riesenhuber1999} \cite{Author2016}
\cite{Balestriero2018}) are partially right. Thus, the group indicator $\bv{h}$
represents the events of $\ca{M}$, of which the expectation is the representative. Let the transport map of $\ca{M}$ be $W: \bv{x}
\mapsto \bv{h}^{T}\bv{W}^{T}\bv{x}$, $\mu^{\ca{H}}(\bv{x} |
\bv{h})$ approximates the measure $\mu^{\ca{W}}(X^{-1}W^{-1}(A)), A \subseteq
\bb{R}^{n}$ in PPMS. That's why instead of using the canonical
coordinate of exponential family distribution, we use its dual affine
coordinate. Though essentially the two coordinate systems are dual views on
the same object, we define the manifold this way to characterize the fact that
for a given NN, $D_{\psi^{*}}$ characterizes the degree
of separation between two events $A, A' \in \Omega$ of which the probability
measures $\mu^{\ca{W}}(A), \mu^{\ca{W}}(A')$ are transported by $W(X(A)),
W(X(A'))$ and approximated by $\mu^{\ca{H}}$. Furthermore, the divergence is
defined by conditioning reflects the fact events can be compared using multiple
criteria, though to evaluate its implication more works are needed. For
how the above definitions relate to classical definitions on NNs
in information geometry, please refer to \cref{sec:related-works-2}.

\expand{As a last remark, in the above definition, we only define the function
  space of $\mu^{\ca{H}}_s(\hat{H}, T | H)$ is a manifold without
touching the overall structure of their interaction, which we believe a much
richer structure lies in, and is an important future direction.}

By a directed application of theorem 14 of \cite{Liese2006}, which is called
information monotony in \cite{amari2016information}, we have

\begin{theorem}[Contraction of divergence between events]
  \label{thm:cg}
  Let $\ca{A} := \ca{A}' \cup \ca{A}''$ be an event collection in
  event space $\Omega$ of PPMS consisting of two event collections, and two
  measurement collection r.e.s $X', X''$ are created for $\ca{A}',
  \ca{A}''$ respectively.
  Let $\tt{S}$ be an S-System, $\ca{M}', \ca{M}''$ be event representations
  built on $X', X''$ by $\tt{S}$ respectively, and $\ca{S}^{\ca{X}}$ be the scale
  poset of $\tt{S}$. Then $\forall s_1, s_2 \in \ca{S}^{\ca{X}}, s_1 < s_2$, we have
  \begin{displaymath}
    D[\bv{\eta}'_{s_1}|_{\bv{h}_1}:\bv{\eta}''_{s_1}|_{\bv{h}_1}] \geq D[\bv{\eta}'_{s_2}|_{\bv{h}_2}:\bv{\eta}''_{s_2}|_{\bv{h}_2}]
  \end{displaymath}
  where $\bv{\eta}'_{s_1}, \bv{\eta}'_{s_2}$ are the neuron coordinates at scale $s_1, s_2$ of
  $\ca{M}'$ respectively; so are $\bv{\eta}''_{s_1}, \bv{\eta}_{s_2}''$ of those of $\ca{M}''$;
  $\bv{h}_{1}, \bv{h}_{2}$ are arbitrarily fixed realizations of group indicator
  r.e.s at scale $s_1, s_2$ respectively.
\end{theorem}
\begin{example}[Contraction of divergence induced by deformation]
  \label{example:deformation}
  Let $g$ be a diffeomorphism group, and $X' = g. X$, the deformed r.e.
  created by applying $g$ on a r.e. $X$. By the above theorem,
  for event representations created by an S-system, coordinated as $\bv{\eta}|_{\bv{h}},
\bv{\eta}'|_{\bv{h}}$, their distance is gradually contracted in term of neuron
divergence. For a review of diffeomorphism group, refer to \cref{sec:manifold}.
\end{example}
The theorem has twofold significance. First, it shows that a recursive
application of CGE would shrink the discrepancy between events, thus possessing
the capacity to contract irrelevant variations in the events, though further
characterizations are needed to give operational guidance. It completes
the incomplete analog between NNs and RG (\cite{Lin2017}), which lacks a
physical quantity to renormalize to produce large scale properties. The
physical quantity is shown as the neuron divergence between event
representations, or more informally, semantic difference between samples, {\it
  though we
note essentially the large scale quantity is group indicator r.e.s of high
scales that represent events that gradually have semantic meaning, of which the
neuron divergence is a property.} Contrary to clear-cut physical quantities like temperature
emerging in physics through RG, {\it a meaningful event group emerges through
learning, which leads us to the section \ref{sec:learn-fram-htms} \ref{sec:optim-landsc-s}.} Second, along with
\cref{def:nn_manifold}, it identifies the proper object if the geometry of NNs is
to be studied. For example, to study the symmetry in the geometry, the object
to investigate is the symmetry in the event space, of which the diffeomorphism
group is a type of symmetry, and invariance is the mapping of events to the
same neuron coordinates. This is in contrast with existing works that study
symmetry by studying the equivariance (\cite{Cohen2016} \cite{Dieleman2016}),
or invariance (\cite{Anselmi2015a}), or linearization of diffeomorphism
(\cite{Mallat2016}) in NNs through studying the changes induced by group
actions in feature maps in the intermediate layers of a sample, which is a
rather ad hoc object. The event space perhaps is the ``mathematical ghost''
lurking in \cite{Mallat2016}.

\subsection{Related works}
\label{sec:related-works-2}

In the related works on hierarchical hypothesis space discussed earlier, all of
them have their own geometry, we only discuss related works in this subsection
that are related to the geometry defined in \cref{sec:contr-effect-coarse}.

Most of the works we are aware of that try to endow a geometry on NNs through
information geometry (IG) were done before the deep learning era, not
surprisingly, by Amari, who developed IG. All the works study NNs with a single hidden
layer. \cite{amari2007methods} formulates the manifold parameterized by all parameters
of a NN as neuromanifold, while in \cref{sec:contr-effect-coarse}, the manifold
we defined focuses on the submanifolds indexed by a scale poset, which will be
discussed more in the next paragraph. Actually, the neuromanifold is the
stochastic manifold consisting of possible probability measure on the random
element set of the event representation built by S-System. Two directions of
analysis have been made. The first is to study the behavior of the curved exponential families
obtained by conditioning, which is done in
\cite{Amari1995}, and falls in the category of generative training. The other is
to study supervised trained NNs, and study the neuromanifold, with a focus on
the impact of singularities on training dynamics (\cite{Amari2006}). The later
proposed the Natural Gradient Descent methods, and many works have been working
on it thereafter, which we will not discuss. The study on the hierarchy has
been limited on decomposition of high order interactions in a single hidden
layer NN (\cite{Amari2001}) without attacking the recursion in NNs, though we
tend to think NNs with more layers unroll higher order interactions, but we do
not find that they pursue this path. As mentioned, the study of hierarchical
behaviors of NNs has been absent, which is the focus of this paper, and is
emphasized in the paragraph below.

The geometry defined in \cref{sec:contr-effect-coarse} is to investigate the
hierarchical geometry of NNs. The compositional exponential family gives the
definition of Neural Network Manifolds that properly identifies the curved exponential
families, or in other words, submanifolds, in a probability family built by
the overall parameter space of NNs, which is complicated, e.g., containing
singularities (\cite{Amari2006}). Note that we do not differentiate submanifolds
with manifolds in the main content to avoid clutters. As discussed, the submanifolds are well
represented by their expectation statistics, and the definition identifies how
coarse graining in divergence happens in \cref{thm:cg}. Thus, definitions
given are distinctive in characterizing the hierarchical geometry of NNs, which
is absent in previous works, which either stay in the realm of single hidden
layer (\cite{Amari1995} \cite{Amari2001}), or take the whole parameter space as
the parameterization of a manifold that contains singularities (which
rigorously is not a manifold) (\cite{Amari2006}), though we are well aware that the
works present in this paper are merely scratching the surface. Our focus for
now is merely to show the coarse graining contraction effect of CGE quantitatively, and
much more works are to be done, e.g., the hierarchical and within-layer
interactions between these submanifolds. As a concrete example, it is known that
the EM algorithm has an IG interpretation (\cite{Amari1995}). The expectation, KL
divergence minimization interpretation of the back propagation algorithm in
\cref{sec:learn-fram-htms} can be interpreted similarly from the IG
perspective. Thus, \cite{Amari1995} can be generalized to NNs with arbitrary
number of layers, and in generative or supervised training settings. It implies the two
directions mentioned in the previous paragraph can be unified, though further
analysis on its impact, e.g., the analysis of singularities, needs more works.

Very recently, at the time of writing this paper, a few reports have been
submitted on the archive that try to attack the supervised deep NNs
(\cite{Amari2018} \cite{Amari2018a}). But again, they follow their old idea that
analyzes the whole neuromanifold. It assumes weights and biases of NNs to be
Gaussian, and study how properties related to the distribution of activations
of each layer change, e.g., fisher information matrix, without trying to
formally define the geometry, or the submanifold structure in the intermediate layers
of NNs.

Lastly, we note that \cite{Lin2017} also tries to discuss the coarse graining
effect in term of the information monotony phenomenon as ``information
distillation'', but it does that rather generally and qualitatively, does not put
the phenomenon in an exact NN context, and not make the connection between it
and the geometry in information.

\section{Learning Framework of S-System (and NNs)}
\label{sec:learn-fram-htms}

\subsection{Learning framework}
\label{sec:learning-framework}

The previous sections defines S-System $\tt{S}$ as a mechanism that given a
measurement collection $X$, $\tt{S}$ uses group coupling probability kernels to
approximate the probability measure transported from $X$ through the transport
maps. Both the transport maps and the kernels are parameterized. To ensure that
the approximation is relevant, a mechanism is needed to determine the
parameters. In other word, how possibly can we ground the probability
approximation created by coupling on ``reality''?  The learning framework that
determines the parameters through optimizing an objective function, is present
in this section. As a remark, supervised discriminative learning and
unsupervised generative learning are two special cases of the learning
objective present, and forward propagation is to maximize the expected data log
likelihood, while back propagation is to minimize the KL divergence.


Let $\ca{M}^{\bar{\ca{W}}}$ be an event representation built by an S-system with
scale poset $\ca{S}^{\ca{Z}}$ built on a measurement collection r.e. $Z := (X, Y)$
with measure $\mu^{\ca{Z}}$ supported on
the PPMS $\ca{W}$. The measurable space $\ca{Z}$ is a product space $\ca{X}
\times \ca{Y}$, where $\ca{X}$ denotes the data space, and $\ca{Y}$ denotes label
space. Let $\ca{M}_s$ be the event representation built at scale $s$. Let
$\mu^{\ca{H}}_{s}, s \in \ca{S}^{\ca{Z}}$ be the probability measure on output
r.e.s $(H_s, \hat{H}_s)$ of $\ca{M}_s$, and $\nu_s$ the law on conditional
group indicator r.e.. The learning of S-System is to minimize the
discrepancy between measure $\mu^{\ca{W}}(X^{-1}(W_s^{-1}(A)))$ and
$\mu^{\ca{H}}_s(A)$ assigned to a event $A \subset \Omega^{\ca{H}_s^{e}}$, where
$\Omega^{\ca{H}_s^{e}}$ is the event space of the probability measure space of $\ca{M}_s$, and $W_s$
is the transport map of the coupling probability kernel of $\ca{M}_s$. One way to characterize the discrepancy is
Maximum Likelihood Estimation (MLE), where the parameters that most likely to
generate the data consist the best estimator. The likelihood function of
$\ca{M}^{\bar{\ca{W}}}$ is
\begin{align*}
  p(X; \bv{\theta})
  =& \sum_{\ca{O}^{\bar{\ca{W}}}\setminus X}p(\ca{O}^{\bar{\ca{W}}}; \bv{\theta})\\
  =& \prod_{s_L \in \ca{S}_{L}}
     \sum_{\bv{h}_{s_L}}\int_{\hat{\bv{h}}_{s_L}}
     \sum_{\bv{h}_{s'}, s' \in p(s_L)}\int_{\hat{\bv{h}}_{s'}, s' \in p(s_L)}
     \sum_{\bv{h}_{s''}, s'' \in p(s')}\int_{\hat{\bv{h}}_{s''}, s'' \in p(s')}
     \ldots\\
   &  \mu^{\ca{H}}_{s_L}(H_{s_L}, \hat{H}_{s_L} | H_{p({s_L})}, \hat{H}_{p(s_L)})
    \prod_{s' \in p(s_L)}\mu^{\ca{H}}_{s'}(H_{s'}, \hat{H}_{s'} | H_{p(s')}, \hat{H}_{p(s')})
    \prod_{s'' \in p(s')}
    \ldots
    \mu^{\ca{Z}}(X)
\end{align*}
where $\ca{O}^{\bar{\ca{W}}}$ and $\bv{\theta}$ are the r.e. set and
the parameters of $\ca{M}^{\bar{\ca{W}}}$ respectively, $\ca{S}_L$ denotes the
set of largest element in $\ca{S}^{\ca{Z}}$, and $p(s)$ denotes the elements in $\ca{S}^{\ca{Z}}$ that are
the predecessors of $s$. Depending on whether $H_s, \hat{H}_s$ are discrete or
not, the summation may be changed to integral, vice versus. It could be understood as
getting the marginal probability distribution of $X$ from a factorized
probability of a direct acyclic graph in probabilistic graphical model (PGM)
(\cite{Koller:2009:PGM:1795555}).

Needless to say, the likelihood function is intractable when the r.e.
set $\ca{O}^{\bar{\ca{W}}}$ gets large. Perhaps more importantly, we do not
know $\mu^{\ca{Z}}(X)$, so we do not know $\mu^{\ca{H}}_s$ since it is built on
the transport map applied on $X$. Thus, to make the estimation tractable, and to
faithfully estimate measure on events groups already seen without making
assumptions on $\mu^{\ca{Z}}(X)$, we make the following decomposition of the log
likelihood function to focus on estimating measures on group indicators r.e.s
\begin{displaymath}
  \ln p(\ca{O}^{\bar{\ca{W}}}; \bv{\theta}) = \ca{L}(\{\nu_s\}_{s \in
    \ca{S}^{\ca{Z}}}, \bv{\theta}) + \sum_{s \in \ca{S}^{\ca{Z}}}\KL^s(\nu_s ||
  \mu^{\ca{X}}(H_s | W_s(X)))
\end{displaymath}
where
\begin{align}
  \ca{L}(\{\nu_{s}\}_{s \in \ca{S}^{\ca{Z}}}, \bv{\theta})
  &= \sum_{\bv{h}_{s}, s \in \ca{O}^{\bar{\ca{W}}}\setminus X}
  q(\ca{O}^{\bar{\ca{W}}}\setminus X)
  \ln \frac{\mu^{\ca{Z}}(X)}{q(\ca{O}^{\bar{\ca{W}}}\setminus X)}\nonumber\\
  q(\ca{O}^{\bar{\ca{W}}}\setminus X)
  &=
  \prod_{s_L \in \ca{S}_{L}}
  \mu^{\ca{H}}_{s_L}(H_{{s_L}} | \hat{H}_{s_L})
    \prod_{s' \in p(s_L)}\mu^{\ca{H}}_{s'}(H_{s'} | \hat{H}_{s'})
    \prod_{s'' \in p(s')}
    \ldots\nonumber\\
  &=
  \prod_{s_L \in \ca{S}_{L}}
  \nu_{s_L}
    \prod_{s' \in p(s_L)}\nu_{s'}
    \prod_{s'' \in p(s')}
    \ldots\nonumber\\
  \KL^s(\nu_s || \mu^{\ca{X}}(H_s | W_s(X)))
  &=
    \sum_{H_s} \nu_{s}\ln \frac{\mu^{\ca{X}}(H_s | W_s(X))}{\nu_{s}}\label{eq:kl}
\end{align}
Note that $\mu^{\ca{H}}_{s}(H_{{s}} | \hat{H}_{s})$ is used instead of
$\mu^{\ca{H}}_{s}(H_{{s}} | \hat{H}_{s}, \hat{H}_{p(s)})$ because in CGE,
$H_s$ is conditional independent with previous output r.e.s given
$\hat{H}_s$. $\ca{L}(\{\nu_{s}\}_{s \in \ca{S}^{\ca{Z}}}, \bv{\theta})$ is
called {\it expected data log likelihood}, $\ln p(\ca{O}^{\bar{\ca{W}}};
\bv{\theta})$ the {\it complete data log likelihood} and $\KL^s(\nu_s ||
\mu^{\ca{X}}(H_s | W_s(X)))$ is the {\it KL divergence at scale $s$}
between estimated measure and true measure.

The decomposition has been used widely in PGM
(\cite{Bishop:2006:PRM:1162264}). Successful techniques derived from it have been
invented known as Variation Inference and Expectation Propagation etc. Yet, one
remarkable difference in the above decomposition and existing decomposition is
that here we decompose the probability measure on physical events in APMS
$\bar{\ca{W}}$, and estimate measure $\nu_s$ that aims to approximate the measure
of groups of events in the event space of the PPMS $\ca{W}$, while in existing
decomposition, their approaches are to hallucinate some parametric
probabilistic models on $\ca{O}^{\bar{\ca{W}}}$ (under the context of S-System),
and because the ``exact'' inference is intractable, they use the decomposition
to make the inference tractable. In essence, we are not making any assumptions
on $Z$, but only on how they are supposed to group together, while existing
approaches using the decomposition is solely about making assumptions on $Z$, and how to make the
computation tractable, thus likely leading to significant model biases.

With the above decomposition, we can see what the training of a supervised NN
is. Forward propagation (FP) is to estimate values of group indicator r.e.s by
assigning $H_s$ a value that maximizes the expected data log likelihood
$\ca{L}(\{\nu_{s}\}_{s \in \ca{S}^{\ca{Z}}}, \bv{\theta})$
w.r.t. $q(\ca{O}^{\bar{\ca{W}}}\setminus X)$ through activation function (though
depending on the activation function chosen, it may not always reach the
maximum), while holding $\bv{\theta}$ fixed. BP is to minimize the KL divergence
$\KL^{s}$ at scale $s$ w.r.t. $\bv{\theta}$ whenever there is a supervisory
information/label on $H_s$ supervising how the events are supposed to group,
while holding $q(\ca{O}^{\bar{\ca{W}}}\setminus X)$ fixed.

The decomposition not only includes supervised NNs, but also includes
variational autoencoder (VAE) (\cite{Welling}), where further assumptions on
probability measure of $X$ are assumed. When absent of labels, a normal
distribution on $H_s$ is assumed, thus encouraging each group indicator r.e. to
learn a grouping that is supposed to be disentangled with the rest. When some
labels exist, we recover semi-supervised VAE.

Thus, supervised learning is never something that stands on its own, so is
unsupervised learning. They are two perspectives to look at the same thing, or
they are Yin and Yang of the Tao in Taoism (or The Book of Change), or the thesis and
anti-thesis of dialectics. They are different ways with different
assumptions to get information to approximate the measure of events groups,
e.g., the group indicator r.e.s in S-System, which represents what has been
recognized. Even pure supervised learning can do some unsupervised learning ---
by pushing KL divergence $\KL^s$ at some scales to zero, the expected data log
likelihood will be closer to the complete data log
likelihood. This partly explains the emergence of generic feature in NNs
(though the maximization of $\ca{L}(\{\nu_{s}\}_{s \in \ca{S}^{\ca{Z}}}, \bv{\theta})$
perhaps is the main reason). So pure unsupervised learning can do some
supervised learning --- the maximization of $\ca{L}(\{\nu_{s}\}_{s \in
\ca{S}^{\ca{Z}}}, \bv{\theta})$ leads to a smaller KL divergence. We do not
observe it in an obvious way in experiments because the grouping does
not necessarily concur to the grouping we humans already have. By imposing some
structure on the grouping scheme, e.g., imposing a normal distribution, we can
discover manifolds that groups events that make sense, e.g., facial expression
or digit variations (\cite{Welling}).

Lastly, we note Bayesian aspects can be further included in the learning
framework by endowing assumptions further on the parameter space.

\subsection{Related works}
\label{sec:related-works-3}

The learning framework is an application of a general probability estimation
framework on the particular case of S-System, thus, the reader may find the
learning framework similar with variational inference widely used in existing
probabilistic graphical models. However, the similarity lies in the fact that
both S-System and PGM approximate probability, of which the decomposition of
complete data likelihood is about probability to be estimated, not about specific hypotheses in
use. The difference between S-System and PGM has been detailed in
\cref{sec:mach-learn-parad}. Previously, the BP algorithms have mostly been
viewed as a heuristic tool, instead of having a theoretically rigorous
derivation. The learning framework of S-System shows that the FP and BP are
actually maximizing the complete data likelihood, and are not merely minimizing
the discrepancy between the estimated conditional probability of labels given
data with the true conditional probability through empirical risk minimization,
but also maximizing the expected data likelihood through activation function.

\section{Optimization Landscape of NNs (and S-System)}
\label{sec:optim-landsc-s}


In the previous section, we formulate the learning objective of S-System to
identify the parameters to hierarchically group and approximate measures
faithfully. Yet, we still do not know whether optimization of the objective
function is possible. This section addresses the issue by studying the
optimization landscape of it.

We will show that despite possessing the normally undesirable complexity,
non-identifiability and singularity (\cite{amari2016information}) properties,
S-System could be marvelously powerful. Stating in a more familiar language, the
problem translates to how a many latent variable model is able to learn? This
is the long standing optimization issue of NNs. We aim at investigating the
principle underlying instead of proving the most general case. More
specifically, when a set of boundedness and diversity conditions hold, we show
that a NN can approximate the probability distribution of any binary group
indicator r.e. given empirical measure of the r.e.; or in other words, for a
class of losses, including the hinge loss, and a class of NNs, including MLP
and CNN, we prove that all local minima of the empirical risk function are
global minima with zero loss values.

The problem is formulated as the following.  Let $\ca{M}^{\bar{\ca{W}}}$ be an
event representation built by an S-system $\tt{S}$ on a measurement collection r.e. $Z$ supported on
the PPMS $\ca{W}$; the measurable space $\ca{Z}$ and measure on $Z$ are $\ca{X}
\times \ca{Y}$ and $\mu^{\ca{Z}}$ respectively, where $\ca{X} := \bb{R}^{n}, \ca{Y} :=
\{-1, 1\}$.  Let the scale poset of $\tt{S}$ be a chain,
symbolically represented as an integer set $\ca{S}^{\ca{X}} = \{0, \ldots, L\}, 0 < \ldots <
L$, and $(H_l, \hat{H}_l), l \in \ca{S}^{\ca{X}}$ the output r.e.s of
$\ca{M}^{\bar{\ca{W}}}$. A reward-penalty mechanism is introduced to give
feedback on the ``faithfulness'' of approximated measure $\mu^{\ca{H}}_{L}$ as a
discrepancy measure between $\nu(H_L | H_{\{1, \ldots, L-1\}}, X)$ and $\mu^{\ca{Z}}(Y |
X)$, where $(X, Y) \in \ca{Z}$. {\it We can see that supervised learning actually
uses the group indicator r.e. $\bv{H}_L$ to approximate the grouping of samples
arbitrated by labels. In a certain way, it formalizes semantics.} The problem
formulation is a part of the general learning framework of S-System described
in \cref{sec:learn-fram-htms}, a.k.a., the term
$\nu_{L}\ln\frac{\mu^{\ca{X}(H_{L}|W_{L}(X))}}{\nu_{L}}$ in \cref{eq:kl}. We only
investigate well how the probability of a high level group indicator
r.e. $H_{L}$ can be approximated here.

The problem setting above is principally the same with the formulation
of a binary supervised learning problem in statistic learning theory (SLT). For
a classic-style formulation, the readers may refer to
\cref{sec:stat-learn-theory}. The key insight of SLT is that instead of seeking
a fully probabilistic formulation, the discrepancy can be formulated as an
empirical risk that measures the discrepancy between $H_L$ and $Y$, calculated
on a set of training samples $\{X_i, Y_i\}_{i=1, \ldots, m}$:
\begin{equation}
  \label{eq:risk}
  R(T) = \frac{1}{m}\sum_{i=1}^{m} l(T(X_i; \bv{\theta}), Y_i) = \frac{1}{m}\sum_{i=1}^{m} l(\hat{H}_L(X_i), Y_i)
\end{equation}
where $T$ here denotes the hierarchical transport map built, $l$ is a loss
function, and $\hat{H}_L(X_i)$ is the coupled r.e. built on $X$ at scale
$L$.

Our goal is to investigate the fundamental principle that makes $R(T)$
tractable. To do this, we study the risk landscape by studying the Hessian of
$R(T)$ of a particular class of loss functions, of which the
eigenvalue spectrum dictates whether critical points of $R(T)$ are local
minima, or saddle points. To motivate the class of
losses, observing that
\begin{equation}
  \label{eq:general_loss}
  \frac{d^{2}}{d\bv{\theta}^{2}}l(T(\bv{x}; \bv{\theta}), y) = l''(T(\bv{x}; \bv{\theta}),
  y)\frac{d}{d\bv{x}}T(\bv{x}; \bv{\theta})\frac{d}{d\bv{x}}T(\bv{x}; \bv{\theta})^{T} + l'(T(\bv{x}; \bv{\theta}), y)
  \frac{d^2}{d\bv{x}}T(\bv{x}; \bv{\theta})
\end{equation}
We study the class $\ca{L}_0$ of functions $l$ and the class of NNs $T$, such that
for $l \in \ca{L}_0$, it satisfies: 1) $l: \bb{R} \rightarrow \bb{R}^{+}$, when
$y$ is taken as a constant; 2) $l$ is convex; 3) the second order derivatives
$\frac{d^2}{d\bv{x}^2}l$ is zero; 4) $\min_{x} l(\bv{x}, y) = 0$, while for $T$ it
satisfies: $\dim(T(\bv{x}; \bv{\theta})) = 1$. The restriction allows us to
study the most critical aspect of the risk function of supervised NNs by
making the first term above zero, while the second term a single matrix
(instead of an addition of matrices). The class of $l$ includes important loss
functions like the hinge loss $\max(0, 1 - \hat{H}_LY)$, and the absolute loss
$|\hat{H}_L - Y|$, which were studied in \cite{Choromanska2015} under
unrealistic assumptions. The class of $T$ is the NN with a single output
neuron, which can be written as
\begin{equation}
  \label{eq:nn}
  T(\bv{x}; \bv{\theta}) = \bv{x}^{T}\prod_{i=1}^{L-1}\bv{W}_{i}\dg(\tilde{\bv{h}}_i)\bv{\alpha}
\end{equation}
where $\bv{\alpha}$ is a vector, $\dg(\tilde{\bv{h}}_i)$ is the diagonal matrix
whose diagonal is the estimated value of $H_i$, and the meaning of rest symbols
is the same as those of MLPCGE in \cref{def:mlp_cge}. {\it Notice that both
$\bv{x}$ and $\bv{h}_i$ are realizations or estimation of r.e.s, thus the
Hessian of $T(\bv{x}; \bv{\theta})$ is a random matrix }(\cite{taotopics}), which
implies the Hessian $\bv{H}$ of $R(T)$ is a random matrix created by summing
random matrices, each of which is a gradient $l'$ multiplies the Hessian of
$T$.

{\it Thus, the problem converts to study the eigen-spectrum of a random matrix
$\bv{H}$.} The conversion looks straightforward now, but is actually a major
obstacle that stops \cite{Choromanska2015} \cite{Pennington2017a}, where a
confusion about the source of randomness led them astray (the point is
discussed in detail in \cref{sec:related-works-4}). With the following {\it realistic}
assumptions on $\bv{H}$ (for a discussion on the {\it practicality} of the
assumptions, refer to \cref{sec:cond-nice-behav}), we show that $l$ has a
surprising benign landscape.
Suppose $\bv{H}$ is a $N \times N$ matrix, and let
\begin{displaymath}
  \bv{A} := \bb{E}[\bv{H}], \frac{1}{\sqrt{N}} \bv{W} := \bv{H} - \bv{A},
  \ca{S}[\bv{R}] := \frac{1}{N}\bb{E}_{\bv{W}}[\bv{WRW}]
\end{displaymath}
where the expectation in $\ca{S}$ is taken w.r.t. $\bv{W}$ while keeping
$\bv{R}$ fixed --- it is a linear operator on the space of matrices.
\begin{assumption}[Boundedness]
  \label{a:boundedness}
  1) $\exists C \in \bb{R}, \forall N \in \bb{N}, ||\bv{A}|| \leq C$, where $||\bv{A}||$
denotes the operator norm. 2) $\exists \mu_q \in \bb{R}, \forall q \in \bb{N},
\forall \alpha \in \bb{J}, \bb{E}[|\bv{W}_{\alpha}|^{q}] \leq \mu_q$, where
$\bb{J} = \bb{I} \times \bb{I}$, and $\bb{I} = \{1, \ldots, N\}$.  3) $\exists
C_1, C_2 \in \bb{R}, \forall R \in \bb{N}, \epsilon > 0,
|||\kappa|||^{\iso}_{2} \leq C_1, |||\kappa||| \leq C_2 N^{\epsilon}$; 4)
$\exists 0 < c < C, \forall \bv{T} \succ \bv{0}, c\ N^{-1}\tr {\bv{T}} \preceq
\ca{S}[\bv{T}] \preceq C N^{-1} \tr \bv{T}$.
\end{assumption}
\begin{assumption}[Diversity]
  \label{a:diversity}
There exists $\mu > 0$ such that the following holds: for every $\alpha \in \bb{I}$ and
$q, R \in \bb{N}$, there exists a sequence of nested sets $\ca{N}_k =
\ca{N}_k(\alpha)$ such that $\alpha \in \ca{N}_1 \subset \ca{N}_2 \subset \cdots \subset \ca{N}_R = \ca{N} \subset
\bb{I}, |\ca{N}| \leq N^{1/2 - \mu}$ and $ \kappa(f(\bv{W}_{\bb{I}\setminus \cup_{j}\ca{N}_{n_j +
1}(\alpha_j)}), g_1(\bv{W}_{\ca{N}_{n_1}(\alpha_1)\setminus \cup_{j\neq 1}\ca{N}(\alpha_{j})}), \ldots,
g_q(\bv{W}_{\ca{N}_{n_q}(\alpha_q)\setminus \cup_{j\neq q}\ca{N}(\alpha_{j})})) \leq
N^{-3q}||f||_{q+1}\prod_{j=1}^{q}||g_j||_{q+1} $, for any $n_1, \ldots, n_q < R$, $\alpha_1,
\ldots, \alpha_q \in \bb{I}$ and real analytic functions $f, g_1, \ldots, g_q$, where $||||_{p}$
is the $L^{p}$ norm on function space. We call the set $\ca{N}$ of $\alpha$ the {\bf
coupling set} of $\alpha$.
\end{assumption}
\begin{theorem}
  \label{thm:landscape}
  Let $R(T)$ be the risk function defined at \cref{eq:risk}, where the loss
function $l$ is of class $\ca{L}_0$, and the transport map $T$ is a neural network
defined at \cref{eq:nn}. If the Hessian $\bv{H}$ of $R(T)$ satisfies
assumptions \ref{a:boundedness} \ref{a:diversity}, $\bb{E}(\bv{H}) =
\bv{0}$, and $N \rightarrow \infty$, then
\begin{enumerate}
\item all local minima are global minima with zero risk
\item A constant $\lambda_0 \in \bb{R}$ exists, such that the operator norm $||\bv{H}||$
  of $\bv{H}$ is upper bounded by $\bb{E}_m[l'(T(X), Y)] \lambda_0$,
  where $\bb{E}_m[l'(T(X), Y)]$ is the empirical expectation of
  $l'$. It implies the regions around the minima are flat basins, where the
eigen-spectrum of $\bv{H}$ is increasingly concentrated around zero.
\end{enumerate}
\end{theorem}
For the impatient readers, refer to \ref{sec:all-non-zero} for the proof
directly and further discussion of the theorem. The error for finite $N$ is
discussed at the remarks of \cref{thm:rm}. In the rest of this section, we
elaborate and prove the theorem step by step.

\subsection{Hessian of NN Is Inherently A Huge Random Matrix}
\label{sec:hessian-nn-huge}

As explained, to study the landscape of the
loss function, we study the eigenvalue distribution of its Hessian $\bv{H}$ at the
critical points. First, we derive the Hessian $\bv{H}$ of loss function of
class $\ca{L}_0$ composed upon NNs with a single output. For a review of matrix
calculus, the reader may refer to \cite{magnus2007matrix}.

The first partial differential of $l$ w.r.t. $\bv{W}_p$ is
\begin{displaymath}
  \partial l(T\bv{x}, y)
  = l'(T\bv{x}, y)
    \bv{\alpha}^{T}\prod_{j=p+1}^{L-1}(\dg(\tilde{\bv{h}}'_j)\bv{W}_{j}^{T})\dg(\tilde{\bv{h}}_p')
    \otimes \bv{x}^{T}\prod_{i=1}^{p-1}(\bv{W}_i\dg(\tilde{\bv{h}}_i)) \partial \tvec \bv{W}_p
\end{displaymath}
where $\otimes$ denotes Kronecker product. Note that for clarity of
presentation, we use the partial differential the same way as differential is
defined and used in \cite{magnus2007matrix}, i.e., $\partial l$ is a number
instead of an infinitely small quantities, though in the book, partial
differential is not defined explicitly. $\tilde{\bv{h}}'$ is an abuse of
notation for clarity and needs some explanation. Recall that
$\{\bv{h}_i\}_{i=1, \ldots, L-1}$ are the group indicator r.e.s, and
$\{\tilde{\bv{h}}_i\}_{i=1, \ldots, L-1}$ are the estimation of them based on transport
maps. $\tilde{\bv{h}}_i$ is a scalar function, and denote it as
$\tilde{\bv{h}}_i(a)$, where $a$ is the computed input. When computing the
partial differential w.r.t. $\bv{W}_p$, by the chain rule, the differential of
$\tilde{\bv{h}}_i(a)$ w.r.t. $\bv{W}_p$ is $\partial \tilde{\bv{h}}_i(a) / \partial a$. To
avoid introducing too much clutter, we denote $\tilde{\bv{h}}'_i$ as $\partial
\tilde{\bv{h}}_i(a) / \partial a$.

Since $\bv{H}$ is symmetric, we only need to compute the block matrices by
taking partial differential w.r.t. $\bv{W}_q$, where $q > p$ --- taking partial
differential w.r.t. $\bv{W}_p$ again gives zero matrix.
\begin{align*}
  \partial^2 l(T\bv{x}, y)
  =& l'(T\bv{x}, y)\\
  & [(\bv{\alpha}^{T}\prod_{k=q+1}^{L-1}(\dg(\tilde{\bv{h}}''_k)\bv{W}_{k}^{T})\dg(\tilde{\bv{h}}''_q)
   \otimes
    \dg(\tilde{\bv{h}}'_p)\prod_{j=p+1}^{q-1}(\bv{W}_{j}\dg(\tilde{\bv{h}}'_j))\partial \tvec \bv{W}_q
    )^{T}\\
  &  \otimes
    \bv{x}^{T}\prod_{i=1}^{p-1}(\bv{W}_i\dg(\tilde{\bv{h}}_i))]
    \partial \tvec \bv{W}_p\\
  =& l'(T\bv{x}, y)\\
  & (\partial \tvec \bv{W}_q) ^{T}
    [\dg(\tilde{\bv{h}}''_q)\prod_{k=q+1}^{L-1}(\bv{W}_{k}\dg(\tilde{\bv{h}}''_k))\bv{\alpha}
    \otimes
    \prod_{j=p+1}^{q-1}(\dg(\tilde{\bv{h}}'_j)\bv{W}_{j}^{T})\dg(\tilde{\bv{h}}'_p)\\
  &  \otimes
    \bv{x}^{T}\prod_{i=1}^{p-1}(\bv{W}_i\dg(\tilde{\bv{h}}_i))]
    \partial \tvec \bv{W}_p
\end{align*}
where again $\tilde{\bv{h}}''$ is an abuse of notation, it is actually
$\tilde{\bv{h}}' \odot \tilde{\bv{h}}'$, the hamadard product of partial differentials
obtained by taking partial differential w.r.t. $\bv{W}_p$, and
w.r.t. $\bv{W}_q$. The two partial differentials are the same because $\bv{a}$,
the input the $\bv{h}$, is the same throughout.
Notice that the estimation $\tilde{\bv{h}}$ of group indicator r.e.s
$H$ is merely a function of $\hat{H}$. It implies $\tilde{\bv{h}}$ is a realization
of a r.e. created by a deterministic coupling by applying a transport
map on $\hat{H}$. The deeper principles of the estimation will be explained in
\cref{sec:learn-fram-htms}. For now, it suffices to stop with the fact that the
entries of $\bv{H}$ is a random variable. As an example, for estimation done by
ReLU, $\tilde{\bv{h}}$ would be $\tilde{\bv{h}}_i = \max\{0, \hat{\bv{h}}_i\}$. Thus, $\bv{H}$ is an
ensemble of real symmetric random matrix with correlated entries. Denote
\begin{equation}
  \label{eq:block_H}
  \bv{H}_{pq} = l'(T\bv{x}, y)\dg(\tilde{\bv{h}}''_q)\prod_{k=q+1}^{L-1}(\bv{W}_{k}\dg(\tilde{\bv{h}}''_k))\bv{\alpha}
      \otimes
      \prod_{j=p+1}^{q-1}(\dg(\tilde{\bv{h}}'_j)\bv{W}_{j}^{T})\dg(\tilde{\bv{h}}'_p)
      \otimes
      \bv{x}^{T}\prod_{i=1}^{p-1}(\bv{W}_i\dg(\tilde{\bv{h}}_i))
\end{equation}
We have the Hessian of $l$ as
\begin{equation}
  \bv{H} =
  \begin{bmatrix}
    \bv{0}          & \bv{H}_{12}^{T}     & \ldots & \bv{H}_{1L}^{T} \\
    \bv{H}_{12}     & \bv{0}          & \ldots & \bv{H}_{2L}^{T} \\
    \vdots          & \ddots          & \ddots & \vdots      \\
    \bv{H}_{1L} & \bv{H}_{2L} & \ldots & \bv{0}
  \end{bmatrix}
  \label{eq:nn_H}
\end{equation}

\subsection{Eigenvalue Distribution of Symmetry Random Matrix with Slow
  Correlation Decay}
\label{sec:eigenv-distr-symm}

In this section, we show how to obtain the eigen-spectrum of $\bv{H}$ through
random matrix theory (RMT) (\cite{taotopics}). RMT has been born out of the study on
the nuclei of heavy atoms, where the spacings between lines in the spectrum of
a heavy atom nucleus is postulated the same with spacings between eigenvalues of a random
matrix (\cite{Wigner1957}). In a certain way, it seems to be the backbone math of
complex systems, where the collective behaviors of sophisticated subunits can be
analyzed stochastically when deterministic or analytic analysis is
intractable.

The following definitions can be found in \cite{taotopics} unless otherwise noted.

The eigen-spectrum is studied as empirical spectral distribution (ESD) in RMT,
define as
\begin{adef}[Empirical Spectral Distribution]
  Given a $N \times N$ random matrix $\bv{H}$, its {\bf empirical spectral
    distribution} $\mu_{\bv{H}}$ is
  \begin{displaymath}
    \mu_{\bv{H}} = \frac{1}{N} \sum_{i=1}^{N}\delta_{\lambda_i}
  \end{displaymath}
  where $\{\lambda_i\}_{i=1, \ldots, N}$ are all eigenvalues of $\bv{H}$ and
  $\delta$ is the delta function.
\end{adef}

Given a hermitian matrix $\bv{H}$, its ESD $\mu_{\bv{H}}(\lambda)$ can be studied
via its resolvent $\bv{G}$.
\begin{adef}[Resolvent]
  Let $\bv{H}$ be a normal matrix, and $z \in \bb{H}$ a spectral parameter. The
{\bf resolvent} $\bv{G}$ of $\bv{H}$ at $z$ is defined as
\begin{displaymath}
  \bv{G} = \bv{G}(z) = \frac{1}{\bv{H} - z}
\end{displaymath}
where
\begin{displaymath}
  \bb{H} := \{ z \in \bb{C} : \Im z > 0 \}
\end{displaymath}
$\bb{C}$ denotes the complex field, and $\Im$ is the function gets imaginary
part of a complex number $z$.
\end{adef}

$\bv{G}$ compactly summarizes the spectral information of $\bv{H}$ around $z$,
which is normally analyzed by {\it functional calculus} on operators, and
defined through Cauchy integral formula on operators
\begin{adef}[Functions of Operators]
  \begin{equation}
    \label{eq:func_calculus}
    f(T) = \frac{1}{2\pi i}\int_C \frac{f(\lambda)}{\lambda - T}d\lambda
  \end{equation}
 \begin{displaymath}
\end{displaymath}
where $f$ is an analytic scalar function and $C$ is an appropriately chosen
contour in $\bb{C}$.
\end{adef}
The formula can be defined on a range of linear operators
(\cite{dunford1957linear}) (Recall that a linear operator is a mapping whose
domain and codomain are defined on the same field). Since the most complex case involved
here will be a normal matrix, we stop at stating that the formula holds true when $T$ is
a normal matrix.

Resolvent $\bv{G}$ is related to eigen-spectrum of $\bv{H}$ through stieltjes transform of $\mu_{\bv{H}}(\lambda)$.
\begin{adef}[Stieltjes Transform]
  Let $\mu$ be a Borel probability measure on $\bb{R}$. Its {\bf Stieltjes transform}
  at a spectral parameter $z \in \bb{H}$ is defined as
  \begin{displaymath}
    m_{\mu}(z) = \int_{\bb{R}} \frac{d\mu(x)}{x - z}
  \end{displaymath}
\end{adef}
With some efforts, it can be seen that the normalized trace of $\bv{G}$ is
stieltjes transform of eigen-spectrum of $\bv{H}$
\begin{displaymath}
  m_{\mu_{\bv{H}}}(z) = \frac{1}{N}\tr G
\end{displaymath}
For a proof, the reader may refer to proposition 2.1 in \cite{Alt2018a}.
$\mu_{\bv{H}}$ can be recovered from $m_{\mu_{\bv{H}}}$ through the inverse
formula of Stieltjes-Perron.
\begin{lemma}[Inverse Stieljies Transform]
  \label{lm:ist}
  Suppose that $\mu$ is a probability measure on $\bb{R}$ and let $m_{\mu}$ be
  its Stieltjes transform. Then for any $a < b$, we have
  \begin{displaymath}
    \mu((a, b)) + \frac{1}{2}[\mu(\{a\}) + \mu(\{b\})] = \lim_{\Im z
      \rightarrow 0} \frac{1}{\pi} \int_{b}^{a} \Im m_{\mu}(z)d\Re z
  \end{displaymath}
  where $\Re$ is the function that gets the real part of $z$. The proof can be
  found at \cite{taotopics} p. 144.
\end{lemma}

Consequently, the problem converts to obtain $\bv{G}$ if we want to obtain
$\mu_{\bv{H}}$. A recent advance in the RMT community has enabled the analysis
of ESD of symmetric random matrix with correlation (\cite{Erdos2017}) from the
perspective of mean field theory (\cite{kadanoff2000statistical}), which we
debrief in the following.

The resolvent $\bv{G}$ holds an identity by definition
\begin{equation}
  \label{eq:resolvent}
  \bv{HG} = \bv{I} + z\bv{G}
\end{equation}
Note that in the above equation $\bv{G}$ is a function $\bv{G}(\bv{H})$ of
$\bv{H}$. When the average fluctuation of entries of $\bv{G}$ w.r.t. to its
mean is small as $N$ grows large, \cref{eq:resolvent} can be turned into a
solvable equation regarding $\bv{G}$ instead of merely a definition. Formally,
it is achieved by taking the expectation of \cref{eq:resolvent}
\begin{equation}
  \label{eq:E_resolvent}
  \bb{E}[\bv{HG}] = \bv{I} + z\bb{E}[\bv{G}]
\end{equation}
When fluctuation of moments beyond the second order are negligible, we can
obtain a class of random matrices whose ESD can be obtained by solving a
truncated cumulant expansion of \cref{eq:E_resolvent}. With the above approach,
using sophisticated multivariate cumulant expansion, \cite{Erdos2017} proves
$\bv{G}$ can be obtained as the unique solution to {\it Matrix Dyson
Equation (MDE)} below
\begin{equation}
  \label{eq:mde}
  \bv{I} + (z - \bv{A} + \ca{S}[\bv{G}])\bv{G} = \bv{0}, \Im \bv{G} \succ \bv{0}, \Im z > 0
\end{equation}
where $\Im \bv{G} \succ 0$ means $\Im \bv{G}$ is positive definite,
\begin{equation}
  \label{eq:mde_def}
  \bv{A} := \bb{E}[\bv{H}],
  \frac{1}{\sqrt{N}} \bv{W} := \bv{H} - \bv{A} ,
  \ca{S}[\bv{R}] := \frac{1}{N}\bb{E}_{\bv{W}}[\bv{WRW}]
\end{equation}
$\ca{S}$ is a linear map on the space of $N \times N$ matrices and $\bv{W}$ is
a random matrix with zero expectation. The expectation is taken
w.r.t. $\bv{W}$, while taking $\bv{R}$ as a deterministic matrix.

We describe their results formally in the following, which begins with
some more definitions, adopted from \cite{Erdos2017}.
\begin{adef}[Cumulant]
  {\bf Cumulants} of $\kappa_{\bv{m}}$ of a random vector $\bv{w} = (w_1,
  \ldots, w_n)$ are defined as the coefficients of log-characteristic function
  \begin{displaymath}
    \log \bb{E}e^{i\bv{t}^{T}\bv{w}} = \sum_{\bv{m}} \kappa_{\bv{m}} \frac{i\bv{t}^{\bv{m}}}{\bv{m}!}
  \end{displaymath}
  where $\sum_{\bv{m}}$ is the sum over all $n$-dimensional multi-indices
$\bv{m} = (m_1, \ldots, m_n)$.
\end{adef}
To recall, a multi-indices is
\begin{adef}[Multi-index]
a $n$-dimensional multi-index is an $n$-tuple
\begin{displaymath}
  \bv{m} = (m_1, \ldots, m_n)
\end{displaymath}
of non-negative integers. Note that $|\bv{m}| = \sum_{i=1}^{n}m_i$,
and $\bv{m}! = \prod_{i=1}^{n}m_i!$.
\end{adef}
Similar with the more familiar concept {\it moment}, cumulant is also a measure
of statistic properties of r.e.s. Particularly, the $k$-order cumulant $\kappa_{}$
characterizes the $k$-way correlation of a set of r.v.s. The key of insight of
the paper is to identify the condition where a matrix entry $w_{\alpha}, \alpha \in
\bb{I}$ is only strongly correlated with a minority of $\bv{W}_{\bb{I}\setminus
\{\alpha\}}$, and higher order cumulants tend to be weak and not influential in
large $N$ limit. Thus, a proper formulation of the correlation strength is
needed, and is defined as the cumulant norms on entries of $\bv{W}$ in the
following. Given $k$ entries $\bv{W}_{\bv{\alpha}}$ at $\bv{\alpha} = \{\alpha_{i}\}_{i=1, \ldots, k}, \alpha_i \in
\bb{I}$ of matrix $\bv{W}$, where duplication is allowed, denote $\kappa(\alpha_1, \ldots,
\alpha_k) = \kappa(w_{\alpha_1}, \ldots, w_{\alpha_k})$.
\begin{adef}[Cumulant Norms]
  \begin{subequations}
  \begin{gather}
    |||\kappa||| := |||\kappa|||_{\leq R} := \max_{2\leq k \leq R}|||\kappa|||_{k},
    |||\kappa|||_k := |||\kappa|||^{\av}_{k} + |||\kappa|||^{\iso}_{k}\nonumber\\
    |||\kappa|||^{\av}_{2} := ||\, | \kappa(*, *)|\, ||, \label{eq:av2}
    |||\kappa|||^{\av}_{k} := N^{-2}\sum_{\alpha_1, \ldots, \alpha_k}|\kappa(\alpha_1, \ldots, \alpha_{k})|, k \geq 4\\
    |||\kappa|||^{\av}_{3} := || \sum_{\alpha_1}|\kappa(\alpha_1, *, *)|\, || +
    \inf_{\kappa=\kappa_{dd} + \kappa_{dc} + \kappa_{cd} + \kappa_{cc}}
    (|||\kappa_{dd}|||_{dd} + |||\kappa_{dc}|||_{dc} + |||\kappa_{cd}|||_{cd} +
    |||\kappa_{cc}|||_{cc}) \label{eq:av3}\\
    |||\kappa|||_{cc} = |||\kappa|||_{dd}
    := N^{-1}\sqrt{\sum_{b_2, a_3}(\sum_{a_2, b_3}\sum_{\alpha_1}|\kappa(\alpha_1, a_2b_2, a_3b_3)|)^2}\nonumber\\
    |||\kappa|||_{cd} := N^{-1}\sqrt{\sum_{b_3, a_1}(\sum_{a_3, b_1}\sum_{\alpha_2}|\kappa(a_1b_1, \alpha_2, a_3b_3)|)^2},
    |||\kappa|||_{dc} := N^{-1}\sqrt{\sum_{b_1, a_2}(\sum_{a_1, b_2}\sum_{\alpha_3}|\kappa(a_1b_1, a_2b_2, \alpha_3)|)^2}\nonumber\\
    |||\kappa|||^{\iso}_{2} := \inf_{\kappa=\kappa_d + \kappa_c}(|||\kappa_d|||_d + |||\kappa_c|||_c),
    |||\kappa_d|||_d := \sup_{||\bv{x}|| \leq \bv{1}}|| ||\, \kappa(\bv{x}*, \cdot *)||\, ||,
    |||\kappa|||_c := \sup_{||\bv{x}|| \leq \bv{1}} ||\, ||\kappa(\bv{x}*, * \cdot)|| \,||\label{eq:iso2}\\
    |||\kappa|||^{\iso}_{k} := ||\sum_{\alpha_1, \ldots, \alpha_{k-2}} |\kappa(\alpha_1, \ldots, \alpha_{k-2}, *, *)|\,||, k \geq 3\nonumber
  \end{gather}
  \end{subequations}
  where in \cref{eq:av3}, the infimum is taken over all decomposition of
  $\kappa$ in four symmetric functions $\kappa_{dd}, \kappa_{dc}, \kappa_{cd},
  \kappa_{cc}$; in \cref{eq:iso2} the infimum is taken over all decomposition
  of $\kappa$ into the sum of symmetric $\kappa_c$ and $\kappa_d$. The norms
  defined in \cref{eq:av2} and \cref{eq:iso2} need some explanation on the
  notation. If, in place of an index $\alpha \in \bb{J}$, we write a dot
  $(\cdot)$ in a scalar quantity then we consider the quantity as a vector
  indexed by the coordinate at the place of the dot. For example,
  $\kappa(a_1\cdot, a_2b_2)$ is a vector, the $i$-th entry of which is
  $\kappa(a_1i, a_2b_2)$ and therefore the inner norms in \cref{eq:av2}
  indicate vector norms. In contrast, the outer norms indicate the operator
  norm of the matrix indexed by star $(*)$. More specifically, $||A(*, *)||$
  refers to the operator norm of the matrix with matrix elements $A_{ij}$. Thus
  $||\, ||\kappa(\bv{x}*, *\cdot)||\, ||$ is the operator norm $||A||$ of the
  matrix $A$ with matrix elements $A_{ij} = ||\kappa(\bv{x}i, j\cdot)||.$
  $\kappa(\bv{x}b_1, a_2b_2)$ denotes $\sum_{a_1}\kappa(a_1b_1,
  a_2b_2)x_{a_1}$, where $\bv{x}$ is a vector.
\end{adef}
We do not want to explain the cumulant norms beyond what has been said,
considering it is too technically involved and rather a distraction. For
interested readers, we suggest reading the paper \cite{Erdos2017}. Equipped
with the cumulant norms, we would have the assumptions stated in
\cref{a:boundedness} \ref{a:diversity} that make MDE valid.

\begin{remark}
  In \cite{Erdos2017}, the functions $f, g_1, \ldots, g_q$ in \cref{a:diversity} are assumed to be
  functions without any qualifiers. We change it to analytic functions for
  further usage. In the proof of \cref{thm:rm}, the functions are
  only required to be analytic, thus even if the assumptions are changed, the
  conclusion still holds.
\end{remark}

The diversity assumption requires that a matrix entry $w_{\alpha}$ only couples
with a minority of the overall entries, and for the rest of the entries, the
coupling strength does not exceed a certain value
$N^{-3q}||f||_{q+1}\prod_{j=1}^{q}||g_j||_{q+1}$ characterized by cumulants. For
example, suppose $q = 1$, given a entry $\alpha_1$, the assumption essentially
states that the entries in the coupling set $\ca{N}(\alpha_1)$ is not strongly
coupled with the resting of the population $\bv{W}_{\bb{I}\setminus \ca{N}_{n_1 +
1}(\alpha_1)}$. The explanation goes similar as $q$ grows, of which the coupling
strengh is characterized by higher order cumulants. While boundedness assumptions
1)2)3)4) states the expectation of $\bv{H}$ is bounded, moments are finite,
cumulants are bounded for the entries that do strongly couples, and
$\ca{S}[\bv{W}]$ is bounded in the sense of eigenvalues.

When the assumptions satisfies, we have the resolvent $\bv{G}$ of a random
matrix $\bv{H}$ close to the solution to the MDE probabilistically with some
regular properties as the following, adopted in an informal style to ease reading
from \cite{Erdos2017} theorem 2.2, \cite{Helton2007} theorem 2.1, and
\cite{Alt2018a} theorem 2.5.
\begin{theorem}
  \label{thm:rm}
  Let $\bv{M}$ be the solution to the Matrix Dyson Equation \cref{eq:mde}, and
  $\rho$ the density function (measure) recovered from normalized trace
  $\frac{1}{N}\tr \bv{M}$ through Stieljies inverse \cref{lm:ist}. We have
  \begin{enumerate}
  \item
    The MDE has a unique solution $\bv{M} = \bv{M}(z)$ for all $z \in \bb{H}$.
  \item
     $\supp \rho$ is a finite union of closed intervals with nonempty
  interior. Moreover, the nonempty interiors are called the {\bf bulk} of
  the eigenvalue density function $\rho$.
  \item
    The resolvent $\bv{G}$ of $\bv{H}$ converges to $\bv{M}$ as $N \rightarrow \infty$.
  \end{enumerate}
\end{theorem}
\begin{remark}
  \cref{thm:rm}.3 is a probably-approximately-correct type result, where the
  error depends on $N$. We do not present the exact error bound here, for that
  it is rather complicated, and does not help understanding --- since we are not
  working on finer behaviors of NNs with a particular size, and do not need
  such a fine granularity characterization yet. We refer interested readers
  to \cite{Erdos2017} theorem 2.1, 2.2, where the exact error bounds are present.
\end{remark}

\subsection{Diversity Assumption is A Precondition to the Power of NNs and
S-System}
\label{sec:cond-nice-behav}

Before we leverage MDE to obtain the eigen-spectrum of the Hessian $\bv{H}$ of
NNs derived at \cref{eq:nn_H}, we explain the meaning of the assumptions
\ref{a:boundedness}, \ref{a:diversity} in the NN and S-System context, so to
point to the potential of the assumptions to give practical guidance on
training NNs.

Recall that the objective function is the empirical risk function $R(T)$ at
\cref{eq:risk}. Given a set of i.i.d. training samples $(X_i, Y_i)_{i=1,\ldots,m}$,
$R(T)$ is a summation of the i.i.d. random matrices. Formally, reusing the
notation to denote $\bv{H}$ the Hessian of $R(T)$ and $\bv{H}_i$ the Hessian of
$l(T(X_i;\bv{\theta}), Y_i)$, we have
\begin{equation}
  \label{eq:H_decom}
  \bv{H} = \frac{1}{m}\sum_{i=1}^{m}\bv{H}_i
\end{equation}
Acute reader may realize that by the multivariate central limiting theorem
(\cite{klenke2012probability} theorem 15.57), $\bv{H}$ will converge to a
Gaussian ensemble, i.e., a random matrix of which the distribution of entries is
a Gaussian process (GP), asymptotically as $m \rightarrow \infty$. For a GP,
all higher order cumulants are zero, which greatly simplifies the picture, and
gives much clearer meaning on the assumptions \ref{a:boundedness}
\ref{a:diversity} made. In the following, we will explain the practicality, and
also how they may serve as guidance to design and improve NNs, in the
asymptotically large sample limit, which gives a picture that can be described
 using more widely used terms, i.e., mean and covariance.

The practicality of boundedness assumptions is obvious, since we do not want
values to blow up. We only note for the two outliers. First, the lower bound in
4) in boundedness assumption, $c N^{-1} \tr \bv{G} \preceq \ca{S}[\bv{G}]$, which is
not about infinity. It asks the eigenvalues of $\ca{S}[\bv{G}]$ to stay close
to its average value, so to let $\ca{S}[\bv{G}]$ stay in the cone of the
positive definite matrices to ensure the stability of the MDE. It is
essentially a constraint on the interaction of second order cumulants, and is
realizable in a NN, though we are not clear on its physical meaning for the
time being. Second, the boundedness assumption 3), as briefly discussed before,
is a bound that bounds the strength of the entries of $\bv{H}$ that do
correlate, while the diversity assumption is about the weakness of the entries
that are not correlated. The practicality of the former is straightforward. To
see the practicality of diversity assumption in the NN context, first we come
back to the concrete form of Hessian. We rewrite
\cref{eq:block_H} in the following form (the equation should be read vertically)
\begin{subequations}
\label{eq:H_analysis}
  \begin{align}
    \bv{H}_{pq}
    &= l'(T\bv{x}, y)\tilde{\bv{W}}_{q\sim(L-1)}\dg(\tilde{\bv{h}}''_{L-1})\bv{\alpha}
    &=&\ \tilde{\bv{\alpha}}_{q}\label{eq:H_1}\\
    & \;\;\;\otimes \tilde{\bv{W}}_{p\sim(q-1)}
    &\ &\otimes \tilde{\bv{W}}_{p\sim(q-1)}\label{eq:H_2}\\
    & \;\;\;\otimes \bv{x}^{T}\tilde{\bv{W}}_{1\sim (p-1)}
    &\ & \otimes \tilde{\bv{x}}^{T}_{p-1}\label{eq:H_3}
  \end{align}
\end{subequations}
where
\begin{align*}
  &\tilde{\bv{W}}_{q\sim(L-1)} :=
  \dg(\tilde{\bv{h}}''_q)\prod_{k=q+1}^{L-2}(\bv{W}_{k}\dg(\tilde{\bv{h}}''_k))\bv{W}_{L-1},
  &\tilde{\bv{\alpha}}_q :=
    \tilde{\bv{W}}_{q\sim(L-1)}\dg(\tilde{\bv{h}}''_{L-1})\bv{\alpha}l'(T\bv{x}, y)\\
  &\tilde{\bv{W}}_{p\sim(q-1)} :=
  \prod_{j=p+1}^{q-1}(\dg(\tilde{\bv{h}}'_j)\bv{W}_{j}^{T})\dg(\tilde{\bv{h}}'_p),
  &\\
  &\tilde{\bv{W}}_{1\sim (p-1)} := \prod_{i=1}^{p-1}(\bv{W}_i\dg(\tilde{\bv{h}}_i)),
  &\tilde{\bv{x}}^{T}_{p-1} := \bv{x}^{T}\tilde{\bv{W}}_{1\sim (p-1)}
\end{align*}
With some efforts, using NN terminologies, it can be viewed that \cref{eq:H_1}
is a vector $\tilde{\bv{\alpha}}_q$ created by back propagating the vector
$\dg(\tilde{\bv{h}}''_{L-1})\bv{\alpha}l'(T\bv{x}, y)$ to layer $q$ by left multiplying $\tilde{\bv{W}}_{q\sim(L-1)}$--- note that if you replace $\bv{h}''_k$
with $\bv{h}'_k$, you get the {\it back propagated gradient}; \cref{eq:H_2}
is the {\it covariance matrix without removing the mean} between neurons at layer $p$ and layer $q-1$,
when taking expectation w.r.t. samples, i.e., $\bb{E}_{z \sim
\mu^{\ca{Z}}}[\tilde{\bv{W}}_{p\sim(q-1)}]$; \cref{eq:H_3} is the {\it forward
propagated activation} at layer $p-1$. Now it is quite clear what the
correlation between entries of the Hessian is about. It is the correlation
between {\it the product of forward propagated neuron activation at layer
$p-1$, the back propagated ``gradient'' at layer $q$, and the strength of
activation paths that connects the two sets of neurons.}

When $\bv{H}$ is a Gaussian ensemble, all higher order cumulants vanishes, thus
the diversity assumption is solely about the second order cumulants, and the
case when $q=1$ and $q=2$. Since $f, \{g_i\}_{i=1,\ldots, q}$ are analytic, $\kappa(f(\cdot),
g_1(\cdot))$ is a generalized cumulant (\cite{mccullagh1987tensor} Chapter 3), which
can be decomposed into a sum of cumulants of entries of the Hessian. So is
$\kappa(f(\cdot), g_1(\cdot), g_2(\cdot))$. Considering that only first and second cumulants
exist, which are means and covariance respectively, $\kappa(f(\cdot), g_1(\cdot))$ thus is a
sum of means of entries of the Hessian and covariance between entries of the
Hessian. So is $\kappa(f(\cdot), g_1(\cdot), g_2(\cdot))$. Using $\kappa(f(\cdot), g_1(\cdot))$ as an
example, the diversity assumption states that for any $\alpha \in \bb{I}$, nested
sets $\ca{N}_1 \subset \ca{N}_2 = \ca{N}, |\ca{N}| \leq N^{1/2 - \mu}$ exist (when only
cumulants up to the second order exist, it suffices to let $R$ be $2$ instead
of every $R \in \bb{N}$ (\cite{Erdos2017})), such that $\kappa(f(\bv{W}_{\bb{I}\setminus
  \ca{N}}), g_1(\bv{W}_{\ca{N}_1})) = \sum_{\beta \in \ca{N}'_1 \subset \ca{N}_1 \cup \bb{I}\setminus \ca{N}}\kappa(\beta) +
\sum_{\beta,\gamma \in \ca{N}' \subset \ca{N}_1 \cup \bb{I}\setminus \ca{N}}\kappa(\beta, \gamma)$, where $\ca{N}'_1,
\ca{N}$ are subsets noted that depends on $f, g_1$. The interpretation is
qualitative. But even from the qualitative interpretation, it can see that {\it
the diversity assumption is on the smallness of the mean and covariance of
$\tilde{\alpha}_{i}\tilde{w}_{jk}\tilde{x}_{l}$, the product of ``gradient'',
activation path correlation strength, and forward activation.} Additionally, to
prove \cref{thm:landscape}, we need a further assumption that $\bb{E}[\bv{H}] =
\bv{0}$. It clearly connects to the experiment tricks used in the community,
such as the early initialization schemes that tries to keep mean of gradient
and activation zero, and standard deviation (std) small (\cite{Glorot2010}
\cite{He}), the normalization schemes that keep the mean of activation zero and
std small (\cite{Ioffe2015}\cite{Salimans2016}), though some more works are
needed to reach there rigorously.

To recap, as stated similarly in \cref{sec:complex-system},
\cref{sec:eigenv-distr-symm}, the diversity assumption states that a diversity
should exist in the neuron population, so that for any neurons, it does not
strongly correlate with the majority of the neuron population. The diversity in
r.e.s. of S-System is not a built-in feature, but a design choice in its
implementations.

Qualitatively, we can see the design of NNs resonates with the diversity
assumption: 1) different group indicator r.e.s. are assumed to be independent
given input r.e.s., referring to \cref{def:mlp_cge} \ref{def:relu}, in which
case, the coupling aims to group the measure that is distinctive w.r.t. other
couplings created through grouping, thus, r.e.s that are not coupled together
are likely to be uncorrelated; 2) activation function creates couplings that
only couple higher scale events with the ``active'' lower scale events, thus
implementing coarse graining that creates events that are composed by different
lower scale events. The above design may not be the only choice, however, it
helps create uncorrelated r.e.s within a scale and across scales, consequently
making the product of the forward propagated activation, activation paths, and
back propagated ``gradient'' tend to be uncorrelated.

Yet, this is a rather general explanation on why diversity occurs without
taking into the finer statistics structure in the data. More improvements may
still be made. For instance, the low correlation existed in CNN is the result
of a coupling that considers the spatial symmetry, where output r.e.s in a
large spatial distance simply does not couples, thus tending to be
uncorrelated.

S-System is a fabulous mechanism that can indefinitely increase the number of
parameters, thus its learning capacity, in a meaningful way, i.e., creating
higher scale coupling yet maintaining the diversity of the r.e.s created. Such
mechanism does not normally hold in other systems or algorithms. Taking linear
NNs for example, though with the potential to infinitely increase its
parameters, matrices that multiply together still have a highly correlation
structure within, thus cannot create a population of diverse neurons that are
of low correlation with a majority of the other neurons. Accompanying the
result we will prove in the next section, which states $R(T)$ can be optimized
to zero, assuming assumptions \ref{a:boundedness} \ref{a:diversity}, we can see
that {\it the diversity assumption actually characterizes a sufficient
precondition to the optimization power of NNs.}

\subsection{NN Loss Landscape: All Local Minima Are Global Minima with Zero Losses}
\label{sec:all-non-zero}

We have obtained the operator equation to describe the eigen-spectrum of the
Hessian of NNs and explained its assumptions. With one further assumption, we
show in this section that for NNs with objective function belonging to the
function class $\ca{L}_0$, all local minima are global minima with zero loss
values.

We outline the strategy first. Since MDE is a nonlinear operator equation, it
is not possible to obtain a close form analytic solution. The only way to get
its solution is an iterative algorithm (\cite{Helton2007}), which is not an easy
task given the millions of parameters of a NN --- remembering that we are dealing
with large $N$ limit --- though it can serve as an exploratory tool. However, we
do are able to get qualitative results by directly analyzing the equation. Our
goal is to show all critical points are saddle points, except for the ones has
zero loss values, which are global minima. To prove it, we prove that at the
points where $R(T) \neq 0$, the eigen-spectrum $\mu_{\bv{H}}$ of the Hessian
$\bv{H}$ is symmetric w.r.t. the $y$-axis, which implies that as long as
non-zero eigenvalues exist, half of them will be negative. To prove it, we
prove the stieltjes transform $m_{\mu_{\bv{H}}}(z)$ of $\mu_{\bv{H}}$ satisfies
$\Im m_{\mu_{\bv{H}}}(-z^{*}) = \Im m_{\mu_{\bv{H}}}(z)$, where $z^*$ denote the complex conjugate of $z$. In
the following, we present the proof formally.
\begin{lemma}
  \label{lm:neg_conjugate}
  Let $\bv{M}(z), \bv{M}'(-z^*)$ be the unique solution to the MDE at spectral
  parameter $z, -z^{*}$ defined at \cref{eq:mde} respectively, and
  $\bv{A} = \bv{0}$. We have
  \begin{displaymath}
    \bv{M}' = -\bv{M}^{*}
  \end{displaymath}
  where $*$ means taking conjugate transpose.
\end{lemma}
\begin{proof}
 First, we rewrite the MDE. Note that $\ca{S}[\bv{G}]$ is positivity
preserving, i.e., $\forall \bv{G} \succ \bv{0}, \ca{S}[\bv{G}] \succ \bv{0}$ by assumption \ref{a:boundedness}
4). In addition, we have $\Im z > 0$, thus $\Im (z + \ca{S}[\bv{G}]) \succ
\bv{0}$. Then, by \cite{Haagerup2005} lemma 3.2, we have $z + \ca{S}[\bv{G}] \succ
\bv{0}$, so it is invertable. Thus, we can rewrite the MDE into the following
form
  \begin{equation}
    \label{eq:mde_alternative}
    \bv{G} = -(z + \ca{S}[\bv{G}])^{-1}
  \end{equation}

  Suppose $\bv{M}$ is a solution to the MDE at spectral parameter $z$.
  The key to the proof is the fact that $\ca{S}[\bv{G}]$ is linear and commutes
  with taking conjugate, thus by replacing $\bv{M}$ with $-\bv{M}^{*}$, and $z$
  with $-z^{*}$, we would get the same equation. We show it formally in the
  following.

  First, note that $\ca{S}[\bv{M}]$ is a linear map of $\bv{M}$, so the we have
  \begin{displaymath}
    \ca{S}[-\bv{M}] = -\ca{S}[\bv{M}]
  \end{displaymath}
  Also, $\ca{S}[\bv{M}]$ commutes with $*$, for the fact
  \begin{displaymath}
    \ca{S}[\bv{M}^{*}] = \bb{E}[\bv{WM^{*}W}] = \bb{E}[(\bv{WMW})^{*}] =
    \bb{E}[\bv{WMW}]^{*} = \ca{S}[\bv{M}]^{*}
  \end{displaymath}

  Furthermore, we $*$ is commute with taking inverse, for the fact
  \begin{align*}
    &&\bv{A}\bv{A}^{-1}                & = \bv{I} \\
    &\implies& (\bv{A}\bv{A}^{-1})^{*} & = \bv{I} \\
    &\implies& \bv{A}^{-1*}\bv{A}^{*}  & = \bv{I} \\
    &\implies& \bv{A}^{-1*}            & = \bv{A}^{*-1}
  \end{align*}

  With the commutativity results, we do the proof. The solution $\bv{M}$
  satisfies the equation
  \begin{displaymath}
    \bv{M} = -(z + \ca{S}[\bv{M}])^{-1}
  \end{displaymath}
  Replacing $\bv{M}$ with $-\bv{M}^{*}$, $z$ with $-z^{*}$, we have
  \begin{align*}
 &          & -\bv{M}^{*} & = -(-z^{*} + \ca{S}[-\bv{M}^{*}])^{-1} \\
 & \implies & \bv{M}^{*} & = -(z^{*} + \ca{S}[\bv{M}^{*}])^{-1} \\
 & \implies & \bv{M}^{*} & = -(z^{*} + \ca{S}[\bv{M}]^{*})^{-1} \\
 & \implies & \bv{M}^{*} & = -(z + \ca{S}[\bv{M}])^{-1*} \\
 & \implies & \bv{M} & = -(z + \ca{S}[\bv{M}])^{-1}
  \end{align*}
  After the replacement, we actually get the same equation. Thus, $-\bv{M}^{*},
  -z^{*}$ also satisfy \cref{eq:mde_alternative}. Since the pair also satisfies the
constrains $\Im \bv{M} \succ 0, \Im z > 0$, and by \cref{thm:rm}, the
solution is unique, we proved the solution $\bv{M}'$ at the spectral parameter
$-z^{*}$ is $-\bv{M}^{*}$.

\end{proof}
\begin{theorem}
  \label{thm:symmetry}
  Let $\bv{H}$ be a real symmetric random matrix satisfies assumptions
\ref{a:boundedness} \ref{a:diversity}, in addition to the assumption that
$\bv{A} = \bv{0}$. Let the ESD of $\bv{H}$ be $\mu_{\bv{H}}$. Then, $\mu_{\bv{H}}$
is symmetric w.r.t. to $y$-axis. In other words, half of the non-zero
eigenvalues are negative. Furthermore, non-zero eigenvalues always exist,
implying $\bv{H}$ will always have negative eigenvalues.
\end{theorem}
\begin{proof}
  By \cref{thm:rm}, the resolvent $\bv{G}$ of $\bv{H}$ is given by the
  the unique solution to \cref{eq:mde} at spectral parameter $z$. Let the solution to
\cref{eq:mde} at spectral parameter $z, -z^{*}$ be $\bv{M}, \bv{M}'$, By
\cref{lm:neg_conjugate}, we have the solutions satisfies
\begin{displaymath}
  \bv{M}' = -\bv{M}^{*}
\end{displaymath}
By \ref{lm:ist}, the ESD of $\bv{H}$ at $\Re z$ is given at
  \begin{displaymath}
    \mu_{\bv{H}}(\Re z) = \lim_{\Im z \rightarrow 0} \frac{1}{\pi}  \Im m_{\mu_{\bv{H}}}(z)
  \end{displaymath}
  Since $m_{\mu_{\bv{H}}}(z) = \frac{1}{N} \tr M$, we have
  \begin{displaymath}
    \mu_{\bv{H}}(\Re z) = \lim_{\Im z \rightarrow 0} \frac{1}{\pi}\frac{1}{N}  \Im  \tr M
  \end{displaymath}
  Similarly,
  \begin{displaymath}
    \mu_{\bv{H}}(\Re (-z^{*}))
    = \lim_{\Im (-z^{*}) \rightarrow 0} \frac{1}{\pi}\frac{1}{N}  \Im \tr M'
  \end{displaymath}
  Note that
  \begin{align*}
 &          & \mu_{\bv{H}}(\Re (-z^{*})) & = \lim_{\Im (-z^{*}) \rightarrow 0} \frac{1}{\pi}\frac{1}{N}  \Im \tr M'\\
 & \implies & \mu_{\bv{H}}(\Re (-z^{*})) & = \lim_{\Im (-z^{*}) \rightarrow 0} \frac{1}{\pi}\frac{1}{N}  \Im \tr (-M^{*})\\
 & \implies & \mu_{\bv{H}}(-\Re z)       & = \lim_{\Im z \rightarrow 0} \frac{1}{\pi}\frac{1}{N}  \Im \tr M
  \end{align*}
  Thus, $\mu_{\bv{H}}(\lambda), \lambda \in \bb{R}$ is symmetric w.r.t. $y$-axis. It follows
  that for all non-zero eigenvalues, half of them are negative.

  By \cref{thm:rm} 2, there are always bulks in $\supp \mu_{\bv{H}}$, thus there
  are always non-zero eigenvalues. Since half of the non-zero eigenvalues are
  negative, it follows $\bv{H}$ always has negative eigenvalues.
\end{proof}

\begin{proof}[Proof of \cref{thm:landscape}]
  First, we prove part 1 of the theorem. The majority of the proof of part 1
have been dispersed earlier in the paper. What the proof here does mostly is to
collect them into one piece.

  The Hessian $\bv{H}$ of the risk function \cref{eq:risk}, can be decomposed
  into a summation of Hessians of loss functions of each training sample, which
  is described in \cref{eq:H_decom}. For each Hessian in the decomposition, it
  is computed in \cref{eq:nn_H}, and it has been shown that $\bv{H}$ is a random
  matrix in \cref{sec:hessian-nn-huge}.

  The analysis of the random matrix $\bv{H}$ needs to break down into two
cases: 1) for all training samples, at least one sample $(x, y)$ has non-zero
loss value; 2) and all training samples are classified properly with zero loss
values.

  We first analyze case 1), since the loss $l$ belongs to function class $\ca{L}_0$, $l$ is
convex and is valued zero at its minimum. When $l(x, y) \neq 0$, we have $l'(x, y) \neq 0$,
thus $\bv{H}$ is a random matrix --- not a zero matrix. The analysis of this
type of random matrix is undertaken in \cref{sec:eigenv-distr-symm}. For a NN,
the assumptions \ref{a:boundedness} \ref{a:diversity} can be satisfied, and the
eigen-spectrum $\mu_{\bv{H}}$ of $\bv{H}$ is given by the MDE defined at
\cref{eq:mde}. The practicality and its potential to guide real world NN
optimization is discussed in \cref{sec:cond-nice-behav}.

  By \cref{thm:symmetry}, $\mu_{\bv{H}}$ is symmetric w.r.t. $y$-axis, and half of
  its non-zero eigenvalues are negative. Thus, for all critical points of
  $R(T)$, its will have half of its non-zero eigenvalues negative. It implies
  all critical points are saddle points.

  Now we turn to the case 2). In this case, all training samples are properly
  classified with zero loss value. Considering the lower bound of $l$ is zero, we
  have reached the global minima. Also, since all critical points in case 1)
  are saddle points, local minima can only be reached in case 2), implying all
  local minima are global minima. Thus, the first part of the theorem is
  proved.

  Now we prove part 2 of the theorem.

  Note that the minima is reached for the fact that we have reached the
situation where the Hessian $\bv{H}$ has degenerated into a zero matrix. Thus,
each local minimum is not a critical point, but an infimum, where in a local
region around the infimum in the parameter space, all the eigenvalues are
increasingly close to zero as the parameters of the NN approach the parameters
at the infimum. We show it formally in the following.

  Writing a block $\bv{H}_{pq}$ (defined at \cref{eq:block_H}) in the Hessian
$\bv{H}_i$ of one sample (defined at \cref{eq:nn_H}) in the form of
  \begin{displaymath}
    \bv{H}_{pq} = l'(T\bv{x}, y) \tilde{\bv{H}}_{pq}
  \end{displaymath}
  where $i$ is the index of the training samples, defined at \cref{eq:risk}.
  Then, putting together $\tilde{\bv{H}}_{pq}$ together to form
$\tilde{\bv{H}}_i$, $\bv{H}_i$ is rewritten in the form of
  \begin{displaymath}
    \bv{H}_i = l'(T\bv{x}_i, y_i) \tilde{\bv{H}}_i
  \end{displaymath}
  Then the Hessian $\bv{H}$ (defined at \cref{eq:H_decom}) of the risk function
  defined \cref{eq:risk} can be rewritten in the form of
  \begin{displaymath}
    \bv{H} = \frac{1}{m}\sum_{i=1}^{m} l'(T\bv{x}_{i}, y_i)\tilde{\bv{H}}_i
  \end{displaymath}
  Taking the operator norm on the both sides
  \begin{displaymath}
    ||\bv{H}|| = ||\frac{1}{m}\sum_{i=1}^{m} l'(T\bv{x}_{i}, y_i)\tilde{\bv{H}}_i||
               \leq \frac{1}{m}\sum_{i=1}^{m} |l'(T\bv{x}_{i}, y_i)|\,||\tilde{\bv{H}}_i||
  \end{displaymath}
  Denote $\max_i\{ ||\tilde{\bv{H}}_i||\}$ as $\lambda_0$, we have
  \begin{align*}
    ||\bv{H}||
    &\leq \frac{1}{m}\sum_{i=1}^{m} |l'(T\bv{x}_{i}, y_i)|\lambda_0\\
    &= \bb{E}_m[l'(TX, Y)] \lambda_0
  \end{align*}
  The above inequality shows that, as the risk decreases, more and more samples
will have zero loss value, consequently $l' = 0$, thus $\bb{E}_m[l']$ will be
increasingly small, thus the operator norm of $\bv{H}$. At the minima where all
$l' = 0$, the Hessian degenerates to a zero matrix.
\end{proof}
\begin{remark}
  It is not necessary for the assumption $\bv{A} = \bv{0}$ to be held for the
  theorem to hold, or more specifically, for $\bv{H}$ to have negative
eigenvalues at its critical points. By proposition 2.1 in \cite{Alt2018a}
  \begin{displaymath}
    \supp \mu_{\bv{H}} \subset \Spec \bv{A} + [-2||\ca{S}||^{1/2}, 2||\ca{S}||^{1/2}]
  \end{displaymath}
  where $\Spec \bv{A}$ denotes the support of the spectrum of the expectation
matrix $\bv{A}$ and $||\ca{S}||$ denotes the norm induced by the operator
norm. Thus, it is possible for the spectrum $\mu_{\bv{H}}$ of $\bv{H}$ to lie at
the left side of the $y$-axis, as long as the spectrum of $\bv{A}$ is not too
way off from the origin. However, existing characterizations on $\supp
\mu_{\bv{H}}$ based on bound are too inexact to make sure the existence of
support on the left of the $y$-axis. To get rid of the zero expectation
assumption, more works are needed to obtain a better characterization, and
could be a direction for future work.
\end{remark}

The the phenomenon characterized by \cref{thm:landscape}.1 is rather
remarkable, if not marvelous. It shows that instead of seeing non-convex
optimization as something to avoid, a class of non-convex objective functions
can be that powerful to the point of ``solving'' --- minimizing the error to the
point of vanishing --- complex problems that nature is dealing with in a rather
reliable fashion. We feel like this is how a brain is doing optimization. We
envision that a much larger class of functions possess such benign loss
landscapes than the one here we have studied. Actually, we have isolated a
function class that represents some of the most essential characteristics of a
more general class of function as shown in \cref{eq:general_loss}, so that we
can show the principle underlying. That is, diverse yet cooperative subunits
aggregating together to form a system can optimize an objective consistently. This
larger class of function could be as important as the concept of convexity, and
would play an important role in optimization. The goal of the paper is to lay
the backbone of the theory of the NNs that make the principles underlying
clear, instead of presenting the theory in its complete form in one go. Thus,
essential properties of the function class are yet to be identified, and will
be part of our future work.

The theorem also contributes to explaining why depth is crucial. The large $N$
limits of the Hessian can be achieved by adding more layers (in the terminology
of S-System,using a scale poset having a longer chain as a subset), even though
the number of neurons in each of the layers may be quite small compared with
the overall number of neurons. The diversity of neurons is possible due to
activation functions (in the terminology of S-System, conditional
grouping extension on estimated realizations of previous created output r.e.s.).

The phenomenon characterized by \cref{thm:landscape}.2 explains why there are
two phases in the training dynamics of NNs, i.e., the rapid error decreasing
phase when loss value is high, and the slow error decaying phase when the loss
value is close to minima. As the error decreases, the expectation of the
derivative of loss values in $R(T)$ will
increasingly approach zero, thus the $\supp \mu_{\bv{H}}$ will concentrate
around zero increasingly, making the landscape increasingly flat and the
training process slowly. It probably also explains why we need to gradually
decrease the step size in the gradient descent algorithms in practice. Very
likely the flat regions are of a small volume compared with the overall parameter
space. Thus, if the training goes conservatively, and inches towards the global
minima, the risk will gradually decrease. But if we give a powerful kick to the
training that induces a large shift in the parameter space, it may kick the
current parameter out of the flat region that can inch toward the global
minima, like kicking a ball from a valley to another mountain in the
hyperspace, thus making the training starts all over again to find a valley to
decrease the risk. A further characterization of the landscape goes beyond
infinitesimal local regions may rigorously prove the conjecture. It even poses
the possibility to move across the flat region rather swiftly, as long as we figure
out how to stay in the valley as we stride big.

\subsection{Related works}
\label{sec:related-works-4}

Similar to the study on the hierarchical hypothesis of NNs, the study on
optimization gains its moment rather recently. We focus on the works that
attack the full complexity of optimization problem of deep NNs, while for more
related works, we refer the readers to related works discussed in
\cite{Dauphin2014} \cite{Nguyen2017} \cite{Liang} for works before the deep
learning era, on shallow networks and NP-hardness of NN optimization.

Roughly, two approaches have been taken in analyzing the optimization of NNs,
one from the linear algebra perspective, the other from mean field theory using
random matrix theory. Our work falls in the latter approach. The linear algebra approach, as
the name suggests, shies away from the nonlinear nature of the
problem. \cite{Kawaguchi2016} proves all local minima of a deep linear NN are
global minima when some rank conditions of the weight matrices are
held. \cite{Nguyen2017} \cite{Nguyen2017b} prove that if in a certain layer of
a NN, it
has more neurons than training samples, which makes it possible that the
feature maps of all samples are linearly independent, then the network can
reach zero training errors. A few works following in the linear-algebraic
direction (\cite{Laurent2017} \cite{Liang} \cite{Yun2017}) improve upon the two
previous results, but using essentially the same approach. As the conditions in
\cite{Kawaguchi2016} indicate, the rank related linear algebraic condition does
not transport to nonlinear NNs. While for \cite{Nguyen2017}, it characterizes a
phenomenon that if in a layer of a NN, it can allocate a neuron to memorize
each training sample, then based on the memorization, it can reach zero
errors. In a certain way, we believe NNs are doing certain memorization, for
the fact that the output elements in the intermediate event representations are
learning template/mean of events, as discussed in
\cref{sec:contr-effect-coarse}. However, it does it in a smart way, where
the templates are decomposed hierarchically. Thus, it is likely we do not need
so many linearly independent intermediate features, which would lead to poor
generalization. Thus, to truly understand the optimization behavior of NNs, we
need to step out of the comfort zone of linearity.

The mean field theory approach using the tools of random matrix theory can
attack the optimization of NNs in its full complexity, though existing works
tend to be confused on the source of randomness. Due to an inadequate
understanding of the randomness induced by activation function,
\cite{Choromanska2015} tries to get rid of the group indicator r.e.s. by
assuming that its value is independent of the input r.e.s. of CGE, which is
unrealistic (\cite{Choromanska2015a}), nevertheless it is a brave attempt, and the
first paper to attack a deep NN in its full complexity. After
\cite{Choromanska2015} which approaches by analogizing with spin glass systems ---
it is a complex system, as NNs are --- some researchers start to study NNs from
 mean field theory from the first principle instead of by analog. Again,
confused with the source of randomness in activation in the intermediate layers
of NNs, \cite{Pennington2017a} just assumes data, weights and errors are of
i.i.d. Gaussian distribution, which are mean field approach assumptions and
unrealistic, and proceeds to analyze the Hessian of the loss function of NNs,
though due to limitations of their assumptions, they can only analyze a NN
with one hidden layer. By laying a theoretical foundation of NNs, S-System
accurately points out where randomness arises in NNs, and what reminds to prove
\cref{thm:landscape} is to find the right random matrix tools.

\newpage
\bibliography{library,url,books}
\bibliographystyle{plainnat_nourl}

\newpage
\appendix
\section*{Appendices}
\label{sec:appendices}
\addcontentsline{toc}{section}{Appendices}

All definitions present in the appendices are adopted and reproduced from
existing literature with sources cited, for the purpose of making exact the
terminology used in the paper.

\section{Coupling Theory}
\label{sec:coupling-theory-1}

The following definitions are adapted from \cite{CIS-9904} unless otherwise
noted.

\begin{adef}[Random Element; Random Variable; Random Function]
  A random element in a measurable space $(E, \mathcal{E})$ defined on a
  probability space $(F, \mathcal{F}, \mu)$ is a measurable mapping $T$ from $(F,
  \mathcal{F}, \mu)$ to $(E, \mathcal{E})$, where
  \begin{displaymath}
    T^{-1}A \in \mathcal{F}, A \in \mathcal{E}; T^{-1}A := \{w \in F: T(w) \in A\}
  \end{displaymath}
  We say $T$ is {\bf supported by} probability measure space $(F, \ca{F}, \mu)$, $(F, \ca{F}, \mu)$ is the
support of $T$, $T$ is an {\bf $\ca{F}/\ca{E}$ measurable mapping} from $F$ to
$E$, and the induced measure $\mu(T^{-1}(A))$ is the {\bf law} of $T$.  Some
r.e.s have special names.  When $(E, \mathcal{E})$ is the measurable
space $(\mathbb{R}, \mathcal{B}(\mathbb{R}))$, where $\mathbb{R}$ is the real
number field and $\mathcal{B}(\mathbb{R})$ is the Borel set, $X$ is also named
as a {\bf random variable}, whose abbreviation is r.v.; when $(E, \mathcal{E})$ is
the multivariate real measurable space, $T$ is named as a multivariate r.v.;
when $(E, \mathcal{E})$ is a function space satisfying certain conditions
(Ionescu–Tulcea theorem, or Kolmogorov's extension theorem
(\cite{klenke2012probability})), $T$ is named as a {\bf random function}.
\end{adef}

\begin{adef}[Coupling]
  A probability measure $\mu$ on $\bigotimes_{i \in \bb{I}}(E_i, \mathcal{E}_{i})$
  is a coupling
  of $\mu_i, i \in \bb{I}$, if $\mu_i$ is the $i$th marginal of $\mu$, that is, if $\mu_i$ is
  induced by the $i$th projection mapping:
  \begin{displaymath}
    \mu(\{x: x_i \in A\}) = \mu_i(A), A \in \mathcal{E}_i, i \in \bb{I}
  \end{displaymath}
  where $\bb{I}$ is an index set, $\mu_i$ is a probability measure on a measurable space $(E_i,
\ca{E}_i)$, and $\bigotimes_{i \in \bb{I}}(E_i, \mathcal{E}_{i}) := (\prod_{i \in
\bb{I}}E_i, \bigotimes_{i \in \bb{I}}\ca{E}_i)$, $\prod_{i \in \bb{I}}E_i$ is the
Cartesian product of $E_i$ and $\bigotimes_{i \in \bb{I}}\ca{E}_i$ is the product
$\sigma$-algebra.
\end{adef}
The general idea of coupling is to find a dependence structure (joint
distribution) from fixed marginal distributions that fits one's purpose.
\begin{adef}[Coupling of Random Element; Deterministic Coupling; Transport Map]
  Given two r.e.s $X, Y$, a coupling $(X, Y)$ refers to the coupling
of probability measure $\mu, \nu$ of probability space $(E, \ca{E}, \mu), (F, \ca{F},
\nu)$, where $\mu, \nu$ is the probability measure of r.e. $X, Y$
respectively. The coupling $(X, Y)$ is called {\bf deterministic} if there exists a
measurable function $T: \ca{E} \to \ca{F}$ such that $Y = T(X)$. $T$ is normally
referred as {\bf transport map}. Informally, we say that $T$ transports measure
$\mu$ of $X$ to measure $\nu$ of $Y$. The definitions are adapted from \cite{villani2008optimal}.
\end{adef}
To recognize an event $\omega_s$ that is composed of events of the lowest detectable
scale $\Omega_{s_0}$, the idea is to transport the probability measure of the event
$\omega_s$ through a deterministic coupling, and construct a r.e. that
represents possible states $\omega_s$ may take. We introduce concepts needed in the
following.
\begin{adef}[Extension of Probability Space]
  A probability space $(\bar{F}, \ca{\bar{F}}, \bar{\mu})$ is an extension
of another probability space $(F, \ca{F}, \mu)$, if $(\bar{F}, \ca{\bar{F}},
\bar{\mu})$ supports a r.e. $\xi$ in $(F, \ca{F})$ having law $\mu$. If $T$ is a r.e. in $(E, \ca{E})$
defined on $(F, \ca{F}, \mu)$, then it has a {\bf copy} $\bar{T}$, i.e., the r.e. $\bar{T}$ defined on
$(\bar{F}, \bar{\ca{F}}, \bar{\mu})$ by $\bar{T}(\bar{\omega}) = T(\xi(\bar{\omega})),
\bar{\omega} \in \bar{\Omega}$ and $\bar{\mu}(\bar{T}(\cdot)) = \mu(T(\cdot))$. $\bar{T}$ is called
{\bf original} r.e.. New r.e.s may be created, which is called {\bf external} r.e.s.
\end{adef}
The goal of S-System is to create extensions of a measurement collection r.e. $\ca{X}$ to reconstruct $\Omega_{s}$ for
some $s \in \ca{S}$ in PPMS, such that $\mu^{\ca{X}}(H_s = h_s) =
\mu^{\ca{W}}(\omega_s)$, where $H_s$ is a r.e. created through extension, and
$h_s$ is a realized value of it.  However, we do not possess the p.d.f. of $X$,
which we have to rely on realizations of $X$. What we can do is to leverage and
only leverage available information through conditioning
extension. \expand{Furthermore, it is likely for limitation of capacity and
economy considerations, S-System has to group together events, thus discretizing
them.}
\begin{adef}[Probability Kernel]
  Given two measurable space $(E_1, \ca{E}_1), (E_2, \ca{E}_2)$, a function
  $Q(\cdot, \cdot)$ from $E_1 \times \ca{E}_2$ to $[0, 1]$ is an $((E_1, \ca{E}_{1}), (E_2,
  \ca{E}_2))$ probability kernel if 1) $Q(\cdot, A)$ is $\ca{E}_{1}/\ca{B}([0, 1])$
  measurable for each $A \in \ca{E}_2$ 2) and $Q(y, \cdot)$ is probability measure on
  $(E_2, \ca{E}_2)$ for each $y \in E_1$.
  A probability kernel  uniquely determines a probability measure on $(E_1,
  \ca{E}_1) \bigotimes (E_2, \ca{E}_2)$ (\cite{ash1972real} Section 2.6.2).
\end{adef}
\begin{adef}[Conditioning Extension]
  Let $T_1$ be an r.e. in $(E_1, \ca{E}_1)$ defined on a probability
  measure space $(F, \ca{F}, \mu)$, and let $Q(\cdot, \cdot)$ be an $((E_1,
\ca{E}_1), (E_2, \ca{E}_2))$ probability kernel. A conditioning extension
$(\bar{F}, \bar{\ca{F}}, \bar{\mu})$ of $(F, \ca{F}, \mu)$ is created by letting
  \begin{align*}
    & (\bar{F}, \bar{\ca{F}}) := (F, \ca{F}) \bigotimes (E_2, \ca{E}_2) & \bar{\mu}(A \times B) := \int_{A} Q(T(\omega), B)\mu(d\omega), A \in \ca{F}, B \in \ca{E}_2\\
    & \xi(\omega, t) := \omega, \omega \in \Omega, t \in E_2 & \bar{T}_1(\omega, t) := T_1(\omega), \omega \in \Omega, t \in E_2\;\; \bar{T}_2(\omega, t) := t, t \in E_2
  \end{align*}
  $\bar{T}_1$ is the original r.e.s in the new probability space,
  while $\bar{T}_2$ is a new external r.e. created. Conditioning extension can be
  repeated countably many times (\cite{ash1972real} Section 2.7.2).
\end{adef}

\begin{adef}[Splitting Extension]
  Let $T_1, T_2$ be r.e.s in $(E_1, \ca{E}_1), (E_2, \ca{E}_2)$
  respectively, defined on a probability space $(F, \ca{F}, \mu)$. Let $\nu$ be a
  probability measure on a Polish space $(E_3, \ca{E}_3)$, let $Q(\cdot, \cdot)$ be an
  $((E_3, \ca{E}_3), (E_2, \ca{E}_2))$ probability kernel, and suppose
  \begin{displaymath}
    \mu(T_2 \in
  A) = \int Q(t, A)\nu(dt), A \in \ca{E}_2
  \end{displaymath}
  Then a splitting extension of $(F,
  \ca{F}, \mu)$ is to create an extension to support a r.e. $T_3$ in
  $(E_3, \ca{E}_3)$ having distribution $\nu$, and such that
  \begin{displaymath}
    \mu(T_2 \in \cdot | T_3 = t) = Q(t, \cdot), t \in \ca{E}_3
  \end{displaymath}
  Furthermore, $T_1$ is conditionally independent of $T_3$ given $T_2$.
\end{adef}
\expand{
  \begin{remark}
    The extension is called a splitting extension for it splits the probability
measure on $(E_2, \ca{E}_2)$ into many groups of measure identified by $t$.
  \end{remark}
}

\section{Manifold and Diffeomorphism Group}
\label{sec:manifold}

The following definitions have been adapted from \cite{lee2012introduction}
unless otherwise noted.

\begin{adef}[Topological Manifold]
  Suppose $M$ is a topological space. We say that $M$ is a {\bf topological
    manifold of dimension n} or a {\bf topological n-manifold} or just a {\bf
n-manifold} if it has the following properties:
  \begin{itemize}
  \item $M$ is a {\bf Hausdorff space}: for every pair of distinct points $p, q
    \in M$, there are disjoint open subsets $U, V \subseteq M$ such that $p \in U$ and $q \in
    V$.
  \item $M$ is {\bf second-countable}: there exists a countable basis for the
    topology of $M$.
  \item $M$ is {\bf locally Euclidean of dimension n}: each point of $M$ has a
    neighborhood that is homeomorphic to an open subset of $\bb{R}^{n}$.
  \end{itemize}
\end{adef}

\begin{adef}[Smooth Mapping; Diffeomorphism]
  If $U, V$ are open subsets of Euclidean spaces $\bb{R}^{n}$ and $\bb{R}^{m}$,
respectively, a function $f: U \rightarrow V$ is said to be {\bf smooth} (or {\bf
  $C^{\infty}$}, or {\bf infinitely differentiable}) if each of its component
functions has continuous partial derivatives of all orders. If in addition $f$
is bijective and has a smooth inverse map, it is called a {\bf diffeomorphism}.
\end{adef}

\begin{adef}[Chart; Coordinate Chart; Smooth Compatible]
  Let $M$ be a topological $n$-manifold. A {\bf coordinate chart} (or just a
  {\bf chart}) on $M$ is a pair $(U, \psi)$, where $U$ is an open subset of $M$
  and $\psi: U \rightarrow \hat{U}$ is a homeomorphism from $U$ to an open subset $\hat{U} =
  \psi(U) \subseteq \bb{R}^{n}$. $U$ is called a {\bf coordinate domain}, or just {\bf
domain}, $\psi$ is called a {\bf (local) coordinate map}, and the component
functions $(x^{1}, \ldots , x^{n})$ of $\psi$, defined by $p \in M, \psi(p) = (x^{1}(p), \ldots ,
x^{n}(p))$, are called {\bf local coordinates} on $U$. Two $(U, \phi), (V, \psi)$ are
called {\bf smoothly compatible} if either $U \cap V = \emptyset$, or $\psi \circ \phi^{-1}$ is a diffeomorphism.

\end{adef}

\begin{adef}[Atlas; Smooth Atlas; Maximal Atlas]
  Let $M$ be a topological manifold. An {\bf atlas} $\ca{A}$ for $M$ is a
collection of charts whose domain cover $M$. If any two charts in $\ca{A}$ is
smoothly compatible with each other, it is called a {\bf smooth atlas}. A
smooth atlas $\ca{A}$ on $M$ is {\bf maximal} if it is not properly contained
in any larger smooth atlas.
\end{adef}

\begin{adef}[Smooth Manifold]
  A {\bf smooth manifold} is a pair $(M, \ca{A})$, where $M$ is a topological
manifold and $\ca{A}$ is a maximal smooth atlas on $M$. When no confusion
exists, we may just say ``$M$ is a smooth manifold''.
\end{adef}

\begin{adef}[Riemannian Metric]
  A {\bf Riemannian metric} $g$ of a smooth manifold $M$ is a symmetric
covariant 2-tensor field on $M$ that is positive definite at each point. It
defines an inner product on $M$, which informally, could be represented by a
quadratic form $\bv{\eta}^{T}\bv{G}\bv{\eta'}$, where $\bv{G} = (g_{ij})$ is a matrix
and $\bv{\eta}, \bv{\eta}'$ are the local coordinates of two points in $M$.
\end{adef}

\begin{adef}[Riemannian Manifold]
  A {\bf Riemannian manifold} is a pair $(M, g)$, where $M$ is a smooth
  manifold and $g$ is a Riemannian manifold. Or in short, we could say $M$ is a
  Riemannian manifold if $M$ is understood to be endowed with a specific
  Riemannian metric.
\end{adef}

The following definition is adopted from \cite{banyaga1997structure}.

\begin{adef}[$C^{\infty}$ Diffeomorphism Group]
  Let $C^{\infty}(M, N)$ denote the space of all $C^{\infty}$ mapping $f: M
  \rightarrow N$, where $M, N$ are smooth manifolds. The diffeomorphism group,
  denoted by $\Diff^{\infty}(M)$, is the set of all $C^{\infty}$
  diffeomorphisms of $M$, the group action of which is the mapping composition.
\end{adef}

To make the definition more concrete to help understanding, we provide an
example adopted from \cite{Mallat2016}.
\begin{example}
 The diffeomorphism group is the set of
deformation that may be applied to objects, e.g., images, of which we can define
a norm to characterize the deformation. A small diffeomorphism acting on a
function $x(u)$ defined on $\bb{R}^{n}$ can be written as a translation of $u$
by a $g(u)$:
\begin{displaymath}
  g. x(u) = x (u - g(u)), g \in \Diff^{\infty}(\bb{R}^{n})
\end{displaymath}
Note that the smooth condition is not necessary, and is only used to avoid introducing
further definitions.  The diffeomorphism translates points by at most
$||g||_{\infty} = \sup_{u \in \bb{R}^{n}}|g(u)|$. Small diffeomorphism
corresponds to $||\nabla g||_{\infty} = \sup_{u}|\nabla g(u)| < 1$, where
$|\nabla g|$ is the matrix norm of the Jacobian matrix of $g$ at $u$. Thus, an
element in a subset of $\Diff^{\infty}(\bb{R}^{n})$ can be understood as a
small deformation of images where the deformation is bounded.
\end{example}

\section{Information Geometry}
\label{sec:information-geometry}

The following definitions are adapted from \cite{amari2016information}.

\begin{adef}[Divergence]
  Suppose that $M$ is a $n$-manifold, of which the points have a local
  coordinates system $\eta$. Given two points $\bv{p}, \bv{q} \in M$, the coordinates of which
  are $\bv{\eta}_p, \bv{\eta}_q$ respectively, a {\bf divergence} is a function of $\bv{\eta}_p,
\bv{\eta}_q$, written as $D[\bv{p}:\bv{q}]$ or $D[\bv{\eta}_p:\bv{\eta}_q]$, which satisfies the
following criteria.
  \begin{itemize}
  \item $D[\bv{p}:\bv{q}] \geq 0$.
  \item $D[\bv{p}:\bv{q}] = 0$, if and only if $\bv{p} = \bv{q}$.
  \end{itemize}
  When $\bv{p}$ and $\bv{q}$ are sufficiently close, and $D$ is differentiable, by
    denoting their coordinates by $\bv{\eta}_p$ and $\bv{\eta}_q = \bv{\eta}_p + d\bv{\eta}$, the Taylor
    expansion of $D$ is written as
    \begin{displaymath}
      D[\bv{\eta}_p:\bv{\eta}_q + d\bv{\eta}] = \frac{1}{2}\sum g_{ij}(\bv{\eta}_p)d\bv{\eta}_i\bv{\eta}_j + O(|d\bv{\eta}|^{3})
    \end{displaymath},
    and matrix $\bv{G} = (g_{ij})$ is positive-definite, depending on $\bv{\eta}_p$.
\end{adef}

\begin{adef}[Bregman Divergence]
  Given a convex function $\psi(\bv{\eta})$, a {\bf Bregman divergence} derived from $\psi$ is
a divergence defined as
  \begin{displaymath}
    D_{\psi}[\bv{\eta}:\bv{\eta}'] = \psi(\bv{\eta}') - \psi(\bv{\eta}) - \nabla\psi({\bv{\eta}})^{T}(\bv{\eta}' - \bv{\eta})
  \end{displaymath}
\end{adef}

\begin{adef}[Legendre Dual; Legendre Transform]
  Given a convex function $\psi(\bv{\eta})$, the {\bf Legendre dual} of $\psi$ is the function $\psi^{*}$
  \begin{displaymath}
    \psi^{*}(\bv{\eta}^{*}) = \bv{\eta}^{T}\bv{\eta}^{*} - \psi(\bv{\eta})
  \end{displaymath}
  where $\bv{\eta} = \bv{f}(\bv{\eta}^{*})$ and $\bv{f}$ is the inverse function of
  $\bv{\eta}^{*} = \nabla\psi(\bv{\eta})$. $\psi^{*}$ is a convex function. $\nabla\psi(\bv{\eta})$ is
  called the {\bf Legendre Transform} of $\eta$.
\end{adef}

\begin{adef}[Dual Bregman Divergence]
  Given a convex function $\psi(\bv{\eta})$, and $D_{\psi}$ the Bregman divergence derived by
  $\psi$. Let $\psi^{*}$ be the Legendre dual function of $\psi$, then the Bregman
  divergence $D_{\psi^{*}}$ defined by $\psi^{*}$ is called the {\bf Dual Bregman Divergence}
  derived by $\psi$. We have
  \begin{displaymath}
    D_{\psi^{*}}[\bv{\eta}^{*}:\bv{\eta}'^{*}] = D_{\psi}[\bv{\eta}':\bv{\eta}]
  \end{displaymath}
\end{adef}

\begin{adef}[Exponential Family of Probability Distributions; Stochastic
Manifold; Affine Coordinate System; Dual Affine
Coordinate System]
  The form of {\bf probability distribution of exponential family} is given by the
probability density function
  \begin{displaymath}
    p(\bv{x}; \bv{\theta}) d\bv{x} = e^{(\boldsymbol{\theta}^{T}\boldsymbol{h}(\bv{x}) - \psi(\bv{\theta}))}d \mu(\bv{x})
  \end{displaymath}
  where $\bv{x}$ is a realizable value of a multivariate random variable,
$\bv{h}(\bv{x})$ is a vector function of $\bv{x}$ which are linearly
independent, $\psi(\bv{\theta})$ is the normalization factor, and $\mu$ is the law on
r.v. $\bv{x}$.

Since $\psi$ is a convex function w.r.t. $\bv{\theta}$, the exponential family
distributions is a Riemannian manifold $(M, g)$ with an {\bf affine coordinate
system} $\bv{\theta} = (\theta_1, \ldots, \theta_n) $, and $g$ is given by the second order Taylor
expansion of the Bregman divergence derived from $\psi$. It is called the {\bf
stochastic manifold of exponential family distribution}. $\bv{\theta}$ is called
{natural} or {canonical} parameters. An alternative coordinate system
of $M$ is given by the Legendre transform
\begin{displaymath}
  \bv{\eta} = \nabla\psi(\bv{\theta}) = \int \bv{h}e^{\bv{\theta}^{T}\bv{h}(\bv{x})}d\mu(\bv{x})
\end{displaymath}
of $\bv{\theta}$, which is well known as the expectation parameter in statistics,
and is called {\bf dual affine coordinate system}. Correspondingly, an
alternative Riemannian metric is derived from the Legendre dual of $\psi$. The
Bregman divergence derived is the well known discrepancy measure on probability,
the KL divergence.
\end{adef}

\section{Statistical Learning Theory}
\label{sec:stat-learn-theory}

Assume a sample space $\mathcal{Z} = \mathcal{X} \times \mathcal{Y}$, where
$\mathcal{X}$ is the space of input data, and $\mathcal{Y}$ is the label
space. We use $S_m = \{ s_i = (\bvec{x}_i, y_i) \}_{i=1}^{m}$ to denote the
training set of size $m$ whose samples are drawn independently and identically
distributed (i.i.d.) according to an unknown distribution $P$. Given a loss
function $l$, the goal of learning is to identify a function $f: \mathcal{X} \mapsto
\mathcal{Y}$ in a hypothesis space (a class $\mathcal{F}$ of functions) that
minimizes the expected error
\[
  R(f) = \mathbb{E}_{z \sim P}\left[l\left(f(\bvec{x}), y\right)\right] ,
\]
where $z = (\bvec{x}, y) \in \mathcal{Z}$ is sampled i.i.d. according to
$P$. Since $P$ is unknown, the observable quantity serving as the proxy to the
true risk $R(f)$ is the empirical error
\[
  R_m(f) = \frac{1}{m}\sum\limits_{i=1}^{m}l\left(f(\bvec{x}_i), y_i\right) .
\]



\end{document}